\documentclass[conference]{IEEEtran}
\usepackage[hyperfootnotes=false]{hyperref}

\usepackage{tikz}
\usepackage{amsmath,amsfonts}

\usepackage{filecontents}

\usepackage{framed}
\usepackage{varwidth}
\usepackage{parskip}
 \usepackage[n, advantage, operators, sets,
adversary, landau, probability, notions,
logic, ff, mm, primitives, events]{cryptocode}
\usepackage{amsmath,subcaption,booktabs,xcolor,amsthm,comment,graphics,comment,multirow,amsfonts,amssymb}
\usepackage[font=small,labelfont=bf]{caption}
\usepackage[listings,skins,breakable]{tcolorbox}
\newtheorem{theorem}{Theorem}

\usepackage{mathtools}
\usepackage{algorithm}
\usepackage{algpseudocode}
\usepackage{csquotes}
\usepackage{array}
\usepackage[export]{adjustbox}
\usepackage{adjustbox}
\usepackage{nicefrac}
\usepackage{float}
\usepackage{enumitem}
\theoremstyle{definition}
\newtheorem{defn}{Definition}
\newtheorem{conj}{Conjecture} 
\def\D{\mathbb{D}}
\def\K{\kappa}
\def\R{\lambda}
\def\T{\tau}
\def\MI{\textsf{MI}}
\def\PoL{\textsf{PoL}}
\def\PoR{\textsf{PoR}}
\def\DF{forging}
\def\ML{\textsf{ML}}
\def\PoL{\textsf{PoL}}

\def\calA{\mathcal{A}}
\def\calC{\mathcal{C}}
\def\GMI{\mathcal{G}_{MI}(\cdot)}

\def \m{M_\theta}
\def \mf{M_\theta'}

\def \calP{\mathcal{P}}
\def \x{x^*}
\def \B{B^{(t)}}
\def \Bf{B^{(t)}_\texttt{forge}}
\def \d{d_{\theta}}

\def \s{s_{\mathcal{A}}}
\def \M{\mu}
\def \c{c_{\mathcal{A}}^U}

\newcommand{\highlight}[1]{\textcolor{orange}{#1}}
\usepackage{url}
\usepackage{caption}
\usepackage{subcaption}

\begin{document}

\title{Can Membership Inferencing be Refuted?}
\author{\IEEEauthorblockN{Zhifeng Kong\textsuperscript{\textsection}, Amrita Roy Chowdhury\textsuperscript{\textsection}, Kamalika Chaudhuri}
\IEEEauthorblockA{University of California, San Diego}}
\maketitle
\begingroup\renewcommand\thefootnote{\textsection}
\footnotetext{The first two authors have made equal contributions.}
\endgroup
\begin{abstract}
    Membership inference (\MI) attack is currently the most popular test for measuring privacy leakage in machine learning models. Given a machine learning model, a data point and some auxiliary information, the goal of an \MI~attack is to determine whether the data point was used to train the model. In this work, we study the reliability of membership inference attacks in practice. Specifically, we show that a model owner can plausibly refute the result of a membership inference test on a data point $x$ by constructing a \textit{proof of repudiation} that proves that the model was trained \textit{without} $x$. We design efficient algorithms to construct proofs of repudiation for all data points of the training dataset. Our empirical evaluation demonstrates the  practical feasibility of our algorithm by  constructing  proofs of repudiation for popular machine learning models on MNIST and CIFAR-10. Consequently, our results call for a re-evaluation of the implications of membership inference attacks in practice.
\end{abstract}

Modern, complex machine learning models  are known to leak sensitive training data~\cite{Papernot16}. Membership inference (\MI) attack~\cite{Shokri17} has emerged as the primary metric for measuring privacy leakage for machine learning models. For instance, government organizations
such as the ICO (UK) and NIST (US) currently highlight membership inference as a potential confidentiality violation
and privacy threat to the training data \cite{Murakonda20}. Given a model $\m$ trained on a dataset $D$, a target data point $\x$, and some auxiliary information, the goal of an \MI~attack is to predict whether the data point $\x$ was used to train the model $\m$. A large body of prior work \cite{carlini2021membership,Carlini22, Jayaraman21, Yeom18,MEntr,Watson2021} has developed effective \MI~attacks and  used them to characterize the privacy leakage for different models. Additionally, \MI~attacks have been deployed in practice for privacy audits -- for example, in the privacy testing library of TensorFlow \cite{blog}. 

Given its widespread popularity, analyzing  the reliability or well-posedness of a membership inference attack is of immense practical importance. Suppose an adversary carries out a successful membership inference attack on a data point $\x$ -- is this is sufficient to indisputably argue that $\x$ was indeed used to train the model $\m$? We take a closer look at this problem in this paper and ask the specific question:
\begin{displayquote} Can a model owner refute a membership inferencing claim in practice?\end{displayquote}

The main challenge here is how can a model owner plausibly assert such a repudiation. We propose to do this as follows. Given a model $\m$ trained by stochastic gradient descent (SGD) on a training dataset $D$ and a target data point $\x \in D$, we present a proof that it is  computationally feasible for the model owner to have generated the \textit{same model from a different dataset $D'=D\setminus \{\x\}$ that does not contain the data point $\x$}. We call this a \textit{Proof-of-Repudiation} (\PoR).  The proof of repudiation empowers a model owner to plausibly deny the membership inferencing prediction and present a counter claim that $\x$ is in fact \textit{not} a member. This discredits the predictions of an \MI~attack. Thus, an adversary cannot ``go-to-court" with the prediction of an \MI~attack.  

 For the concrete construction, we use a log that records the entire training trajectory (i.e., the sequence of SGD updates including all the relevant training details, such as mini-batch indices and model weights) starting from the initialization to the final model weights.  Any third-party entity can reproduce the steps of gradient descent from the log and verify the validity of the sequence provided. Specifically, the model owner produces a ``forged" log showing the model's training trajectory on a dataset $D' = D\setminus \{\x\}$ and uses it as a proof of repudiation for refuting membership inference claims on $\x$. This suffices to raise a \textit{reasonable doubt} on the adversary's claim. The above construction is based on the notion of \textit{forgeability} which was introduced by Thudi et al. \cite{Thudi22} in the context of machine unlearning (Section \ref{sec:background:forgeability}). An attractive aspect of our construction is that the generated \PoR~is completely \textit{agnostic} of the \MI~attack. Additionally, it requires \textit{no} modifications to the model's training and inference pipeline. 

In this paper, we aim to produce proofs of repudiation for all the data points of the training dataset. A naive approach is to repeat the construction by Thudi et al. \cite{Thudi22} on every data point. However, there are two challenges with this approach. First, this is extremely computationally inefficient especially for modern models trained on huge datasets. Second, their approach cannot accommodate modern sophisticated training techniques, such as data augmentation. 

We address both challenges through a series of algorithmic innovations. First, we provide a way to strategically reuse certain information -- ``forged" mini-batches as well as proofs of repudiation -- across multiple training data points, which improves efficiency by three orders of magnitude.  Second, we extend support for a popular data augmentation technique known as the Random Flip augmentation by using additional transformation information. Combining these techniques gives us a practical proof-of-repudiation generation algorithm that can be run on popular models on standard datasets, such as MNIST and CIFAR-10, even with academic computational resources. 

We empirically evaluate our algorithm on three popular models trained on two standard datasets -- MNIST and CIFAR-10. Our results indicate that it is possible to generate valid proofs of repudiation for at least $98.8\%$ and $87.8\%$ of the data points of MNIST and CIFAR-10, respectively. We measure the quality of the proofs by evaluating them against five standard \MI~attacks, including state-of-the-art attacks such as LiRA~\cite{Carlini22} and EnhancedMIA~\cite{ye2021enhanced}, which are specifically designed to maximize prediction accuracy under a \textit{low} false positive regime. We see that our algorithm can fool even these attacks, and performs remarkably well in practice.  
Additionally, we analyze the limitations of proofs of repudiation and observe that one cannot generate a valid proof for outlier/out-of-distribution data points. We provide a formal impossibility result for linear models (Theorem \ref{thm:impossibility}). 

Since proofs of repudiation enable a model owner to refute a membership inference claim \textit{post-attack}, our results question the reliability of \MI~attacks in practice, and call for a re-evaluation of how they are used in measuring the privacy leakage of \ML~models. We argue that \MI~attacks are perhaps better suited for distinguishing between in- and out-of-distribution data points in practice (see Sections \ref{sec:PoR:impossibility} and \ref{sec:discussion}). Making reliable inferencing on in-distribution data points would require a new lens of investigation.

\section{Background}\label{sec:background}

\subsection{Machine Learning}
Machine learning (\ML) is the task of learning model $M_{\theta}$ from a dataset where $\theta$ denotes the parameters of the model. Here, we focus on supervised learning where the dataset $D$ consists
of points $(x_i, y_i) \in \mathcal{X}\times \{1,\cdots, c\}$, where $y_i \in \{1,\cdots , c\}$ is the
label of the input $x_i \in D$, and there are $c$ possible labels. The
goal of supervised ML is to predict the label $y$ of an unlabelled
input $x$ by using the knowledge learned from the labeled dataset $D$.  Let $L$ be the loss function.  Let $g$ be the update rule that takes the model $\theta_t$ and a mini-batch $(\hat{x}^{(t)}, \hat{y}^{(t)})$ of size $b$ at step $t$ as inputs, and updates the model parameter as
$\theta_{t+1} = g(\theta_t,(\hat{x}^{(t)}, \hat{y}^{(t)}))$.  In this paper, we focus on the SGD update rule. At each $t$, $(\hat{x}^{(t)}, \hat{y}^{(t)})$ is uniformly selected from $D$ without replacement, and the update rule is 
\begin{gather}\theta_{t+1} = \theta_t-\texttt{step\_size}\cdot\frac{1}{b}\sum_{i=1}^b
\left.\frac{\partial L(M_{\theta}(\hat{x}^{(t)}_i),\hat{y}^{(t)}_i)}{\partial\theta}\right\vert_{\theta=\theta_t}.
\end{gather}
For conciseness, in the rest of the paper, we drop the notation $y$ for labels and represent a sample by $x$, and let $\hat{x}\in D$ denote a mini-batch of size $b$ sampled from $D$. 

\subsection{Membership Inference Attacks}\label{sec:background:MI}
The goal of a membership inference (\MI) attack ~\cite{Shokri17,carlini2021membership,Carlini22, Jayaraman21, Yeom18,MEntr,Watson2021} is to predict whether a data point was used in training a given model. \MI~attacks are currently the most widely deployed attack for auditing privacy of machine learning models. We consider a model owner who has trained a \ML~model $\m$ on a dataset $D$. Both the model $\m$ and the dataset $D$ is proprietary to the model owner. 
\par Prior work~\cite{Carlini22, Jayaraman21, Yeom18} has made attempts at formalizing the \MI~attack along the lines of a security game which is inspired by cryptography. The attack  takes the model $\m$, a data point drawn from the input distribution $\x\sim \D$, possibly some auxiliary knowledge $\psi$ (for instance, query access to the distribution $\mathbb{D}$) and outputs its prediction bit $b\leftarrow \mathcal{A}(\x,\m,\psi)$ indicating its belief about $\x$'s membership.  The security games are played between two parties -- the challenger $\mathcal{C}$ and the adversary $\mathcal{A}$. For the \MI~attack, the model owner acts as the challenger and the security game, $\GMI$, is defined as follows:
\begin{defn}[Membership Inference Security Game $(\mathcal{G}_{\MI}(\cdot))$] The membership inference attack is defined as follows: 
\begin{enumerate}\item  
The challenger $\calC$ samples a dataset $D\sim \mathbb{D}$ via a \textit{fixed} random seed, $S_D$, and trains a model $M_{\theta}\leftarrow\mathcal{T}(D)$ on it by using a fixed random seed, $S_\mathcal{T}$, in the training algorithm $\mathcal{T}(\cdot)$. \item The challenger $\calC$ flips a bit $\mathtt{b}$, and if $\mathtt{b} = 0$, samples a
fresh challenge point from the distribution $\x \sim \mathbb{D}$
(such that $\x \not \in D$). Otherwise, the challenger selects
a data point from the training set $\x \leftarrow D$. The challenger uses a fresh random seed $S_{x^*}$ for both the cases.
\item The challenger $\calC$ sends $\x$ to the adversary $\calA$.
\item Let $\psi$ denote the adversary's auxiliary knowledge (for instance, query access to the distribution $\mathbb{D}$). The adversary
 outputs a bit $\mathtt{b}' \leftarrow \calA(\x,M_{\theta}, \psi)$.
\item Output $1$ if $\mathtt{b}' = \mathtt{b}$, and $0$ otherwise. \end{enumerate}
\end{defn}
In the above game, fixing the random seeds, $S_D$ and $S_\mathcal{T}$ imply that the challenger always trains the \textit{same} model $\m$ from the \textit{same} dataset $D$. We choose this particular setting because this captures the real-world scenario of a typical \ML~model post-deployment (setting 2 as discussed later). There can be a couple of variations of the above game based on the (fixed or fresh) instantiation of the random seeds, $S_D,S_{\mathcal{T}}$, and $S_{x^*}$ (see \cite{ye2021enhanced} for more details).

 Typically, \MI~attacks have been studied under the following two settings:

\textbf{Setting 1.} \MI~attacks are used for \textit{auditing} privacy internally \textit{pre-deployment}. Here, the model owner 
acts as both the challenger and the adversary in the game and carries out \MI~attacks to stress-test the model privacy for data protection impact assessment~\cite{GDPR}. The success rate of the \MI~attacks is used to assess the privacy leakage from $\m$ about its training dataset $D$.  

\textbf{Setting 2.} The second setting deals with practical threat that an \MI~attack poses for a \ML~model \textit{post deployment}. In the above game, the adversary here is a real-world attacker  who tries to infer sensitive membership information and act on it (see Section \ref{sec:discussion} for some illustrative examples). 

Our results have \textit{no} bearing on the first setting of using \MI~attacks for internal privacy audits pre-deployment.  For the rest of the paper, we are only interested in the second setting -- refuting membership inference claims of a real-world \MI~attack.

Note that we consider a white-box \MI~attack construction where we assume that the adversary $\calA$ has access to the full model $\m$ to capture the most general setting. However, restricting the adversary to only a black-box access to the model (for instance, in the case of ML-as-a-Service platforms) does not affect any of the discussion in this paper.

\subsection{Proof-Of-Learning (PoL)} \label{sec:background:PoL}
\begin{figure}\centering \includegraphics[width=\linewidth]{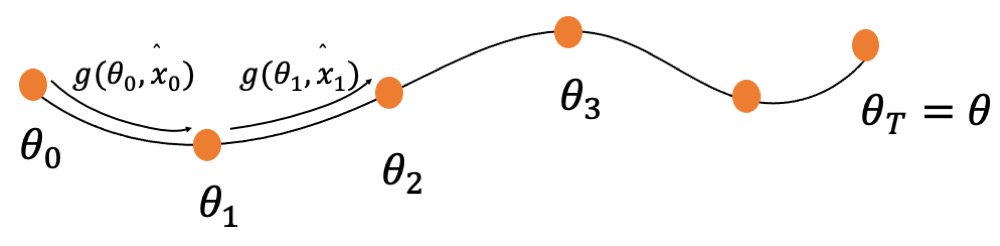}
\caption{PoL~records the training trajectory, i.e., sequence of SGD updates, of a model $\m$.}\label{fig:forgeability 1} \end{figure}

The concept of Proof-of-Learning (\PoL) was introduced by
Jia et al. \cite{Jia21}. It enables an entity to provide evidence that they have \begin{itemize}\item 
 performed the necessary computation on a dataset $D$ to train a model $\m$, and that \item all the steps have been performed correctly. \end{itemize}  Consequently, \PoL~allows the verification of the integrity of the
training procedure used to obtain $\m$ by a third-party authority, such as a regulator. It was proposed in the context of establishing model ownership, or verifiable outsourcing of model training.  The core idea is to maintain a log throughout training which facilitates
reproduction of the alleged computations the entity carried out. A valid \PoL~log is formally defined as follows:

\begin{defn}[Proof-of-Learning ~(\PoL)~\cite{Jia21}] Recall, $g(\theta,\hat{x}) = \theta'$ denotes
updating the model parameters $\theta$ to $\theta'$ where
$g$ is a training algorithm/update rule. A valid  Proof-of-Learning (\PoL) log, $(g, d, \epsilon)$ is a  sequence of $\{(\theta_i
,\hat{x}_i)\}_{i\in J}$ for some countable indexing set $J$, such that $\forall i \in J,  d(\theta_{i+1},g(\theta_i
,\hat{x}_i)) \leq \epsilon$ for
some training function $g$ and metric $d$ on the parameter space.
The threshold $\epsilon$ is a tolerance parameter for the verification. \end{defn}
From the above definition, the \PoL~documents
intermediate checkpoints of the model, data points used, and
any other information required for the updates during training
(for instance, hyperparameters that define the update rule  at each
step).  The verifier checks a \PoL’s validity by reproducing the $t$-th intermediate checkpoint of the
model based on the information given in the log, including
the $(t - 1)$-th checkpoint, data points used at this step, and
the same update rule as defined in the log. Then the verifier computes
the distance between the $t$-th checkpoint in the log and the reproduced $t$-th checkpoint in the parameter space. This distance is called the verification error, and we say this update
is valid if the verification error is below a certain threshold. For the rest of the paper, we use the $\ell_2$-distance as our distance metric $d$. 

\subsection{Forgeability} \label{sec:background:forgeability}

\begin{figure}\centering \includegraphics[width=\linewidth]{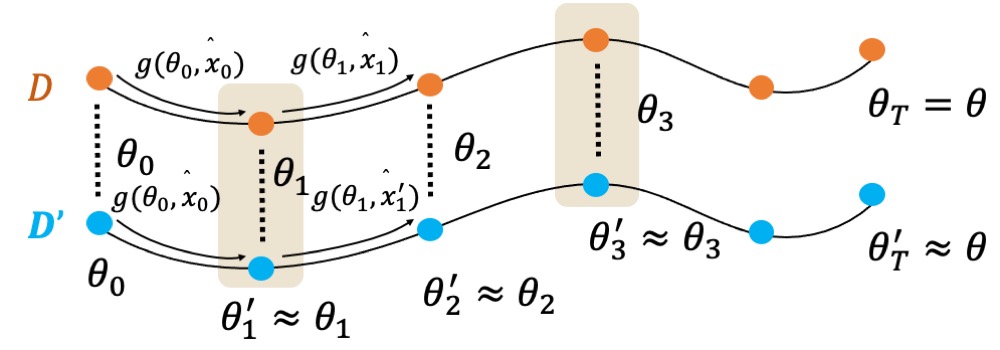}
\caption{Forgeability allows replacing data points from $D$ by points from $D'$ while keeping the model updates computed on these points (almost) identical
.}\label{fig:forgeability 2} \end{figure}

Privacy regulations such as GDPR~\cite{GDPR} and CCPA~\cite{CCPA} empower individuals with the Right to Deletion or Right to be Forgotten wherein a data owner who had previously allowed the use of their personal data can later retract the authorization.   In the context of \ML~models,  this requires the model owner to perform \textit{machine unlearning}~\cite{Bourtoule21} to scrub the effect of the data point from the model.  Machine unlearning is the mechanism of proving to the data owner that the model, which had originally been trained on their data, has been modified to ``forget" whatever it might have learned from the data.  

Forgeability~\cite{Thudi22} was introduced in the context of machine unlearning. In what follows, we present a simplified description of the original proposal which is relevant for understanding our paper. 

\par Two datasets are defined to be forgeable w.r.t a given model if the same model parameters (up to some
small per-step error) can \textit{feasibly} be obtained by using either of the datasets. Here, feasibility is defined in terms of the \PoL~logs.

\begin{defn}[Forgeability~\cite{Thudi22}]  For two datasets $D$ and
$D'$, and \textsf{PoL}s stemming from the same model $\m$, we say $D'$ forges $D$ (with error $\epsilon'$), if we have
\begin{gather} \forall\ i, \exists\ \hat{x}'_i \in D', s.t.
\\ d(\theta_{i+1},g(\theta_i
,\hat{x}'_i)) \leq \epsilon' .
\end{gather}\label{def:forgeability}
\end{defn}
Intuitively, forgeability means that the data points (mini-batches) of $D$ can be mapped to data points of $D'$ such that the parameter space is preserved (upto $\epsilon$) as determined by a \PoL~log. Thus, we can replace  the data points in $D$ with that in $D'$ and still obtain a valid \PoL~for the same model $\m$ (up to an error of $\epsilon'$).  Thus, forgeability is the ability to replace data points from one
dataset by data points from another dataset while keeping the
model updates computed on these data points (almost) identical.

\textbf{Consequences for Machine Unlearning.} In the context of machine unlearning, the authors present an algorithm (Algorithm \ref{alg: forging attack}) to generate \PoL~logs for carrying out the forging attack. The algorithm forges a dataset $D$ with its subset dataset $D\setminus \{x_-\}$ where $x_-\in D$ is the data point to be deleted. For this, the algorithm randomly samples $\M$ number of candidate mini-batches that do not contain $x_-$ (Step 4-5). Then, it chooses the one that minimizes the $\ell_2$-distance between gradients computed from candidate and original mini-batches (Step 6-8). We call this the forged mini-batch, $B^{(t)}_{\texttt{forge}}$ and this step is repeated for every training iteration $t\in [\T]$. All the $\T$ forged mini-batches constitute the forged $\PoL$.
\begin{algorithm*}
    \centering
    \caption{Generation of forged PoL~(\cite{Thudi22})}
    \label{alg: forging attack}
    \begin{algorithmic}[1]
        \State \textbf{Inputs:} $D=\{x_1,\cdots,x_n\}$ - Training dataset;  $\T$ - Total number of iterations; $B_*^{(t)}$ - Training mini-batches  for the $t$-th iteration where $t\in [\T]$; $b$ - Size of each mini-batch, $\left|B_*^{(t)}\right|$; $\theta^{(0)}$ - Initialization parameters; $x_-$ - Data point to be deleted.
        \State \textbf{Hyper-parameters:} $\texttt{step\_size}$;
        $\M$ - Number of candidate mini-batches sampled per iteration for forging.
        \For{$t=1,\cdots,\T$}
            \State Randomly select a subset $D_{\mathrm{sub}}$ from $D\setminus \{x_-\}$.
            \State Randomly select $\M$ mini-batches of size $b$ from $D_{\mathrm{sub}}$: $B_1^{(t)},\cdots,B_\M^{(t)}$.
            \State Compute gradients and update the original model: \\
            $$\texttt{grad}^{(t)} = \frac1b\nabla_{\theta} \sum_{x\in B_*^{(t)}}\ell(x;\theta^{(t)});~~\theta^{(t)}=\theta^{(t-1)}-\texttt{step\_size}\cdot\texttt{grad}^{(t)}.$$
            \State Choose the best mini-batch to approximate the gradient: \\
            $$m_*^{(t)} = \arg\min_{m\in [\M]} \left\|\frac1b\nabla_{\theta} \sum_{x\in B_m^{(t)}}\ell(x;\theta^{(t)})-\texttt{grad}^{(t)}\right\|_2^2.$$
            \State $B_{\texttt{forge}}^{(t)} = B_{m_*^{(t)}}^{(t)}$.
        \EndFor
        \State \textbf{Return:} original model $\theta_*=\theta^{(\T)}$; all forged mini-batches $\{B_{\texttt{forge}}^{(t)}: t\in [\T]\}$.
    \end{algorithmic}
\end{algorithm*}

\section{Proof-of-Repudiation (PoR)}\label{sec:PoR:Formalization}
In this section, we formalize our problem statement. 
\par\textbf{Problem Statement.}  Consider a scenario where a real-world adversary carries out an \MI~attack on a model $\m$ trained on dataset $D$ and correctly predicts that a point $x^*\in D$ is indeed a member of the training dataset. This corresponds to the second setting as discussed in Section \ref{sec:background:MI}. Given the above scenario, the question we study here is that can the model owner now repudiate the adversary's membership inference claim on $\x$. We show that the ability to forge (Section \ref{sec:background:forgeability}) enables the model owner to  construct a \textit{Proof-of-Repudiation} (\PoR) which empowers them to \textit{repudiate} the prediction. The \PoR~enables the model owner to plausibly deny the \MI~prediction and present a counter claim that $\x$ is in fact \textit{not} a member. This discredits the predictions of an \MI~attack.

\subsection{Definition of a PoR} \label{sec:PoR:def}
Next, we provide the formal definition of \PoR.
\begin{defn}[Proof Of Repudiation (\PoR)] Let  $\m$ be a \ML~model trained on a dataset $D$. A valid proof-of-repudiation $\calP(\m,D,\x)$ for a membership inferencing claim on a data point $\x\in D$ is a certificate to the effect that it is computationally \textit{feasible} for the model owner to have obtained the \textit{same} model $\m$ from a \textit{different} dataset $D'$ such that $\x\not\in D'$.\label{def:PoR}
\end{defn}

From the above definition, given a model $\m$ and a membership inference claim that data point $\x$ is a member of its training dataset $D$, a valid \PoR~should provide evidence that the model owner could have obtained the model $\m$ from a different dataset $D'$ that does \textit{not} contain the data point $\x$. Clearly, a certificate to the above effect suffices to raise a \textit{reasonable doubt} on the membership inference claim on $\x$. In other words, the adversary now \textit{cannot} claim with absolute certainty that $\x$ was indeed 
a member of the training dataset $D$. 
Next, we describe a set of desirable properties of a \PoR.
\begin{itemize}
     \item \textbf{P1. Correctness.} The \PoR~should be verifiable with very high probability, i.e., the  reported computation (concerning the training of $\m$) should be reproducible by any third-party verifier.
    \item \textbf{P2. Attack Agnostic.} \PoR~should be completely agnostic of the \MI~attack used for making the prediction, i.e.,  the \textit{same} \PoR~should be valid for \textit{all} \MI~attacks as long as they all have the make the same prediction ($\mathtt{b}=1$) on $\x$.
    \item \textbf{P3. Model Agnostic.} \PoR~should be agnostic of specificities of the \ML~model, such as model architecture, number of parameters. 
      \item \textbf{P4. Proof Generation Efficiency.} A challenger $\calC$ should be able to generate the \PoR~efficiently. Specifically, not too much extra effort  than training the algorithm.
      \item  \textbf{P5. Proof Verification Efficiency. } The \PoR~verification should be efficient and ideally less computationally expensive than generating the \PoR. Additionally, verification
should succeed even on a different hardware than the one used for generation.

\end{itemize}

\section{How to construct a PoR?}\label{sec:PoR:construction}
In this section, we describe the construction of a \PoR. For this, first we introduce a practical relaxation of a \PoR~followed by its construction leveraging forgeability. \par 

For the rest of the paper, we consider the dataset $D'=D\setminus \{\x\}$, i.e., the original dataset with only the data point in contention, $\x$, removed. As indicated in Definition \ref{def:PoR}, we can conclude that an \MI~attack is infeasible for an input data point $x^* \in D$  if we can \textit{prove} that both the datasets $D$ and $D'$ can generate the exact \textit{same} model, $\m$.  However here, we relax this assumption and allow the two models --  the one obtained from training on $D$, $\m$, and the other from $D'$, $\mf$ to differ by a small amount $\epsilon^*$. Specifically, we mean that the parameters of the two models, $\m$ and $\mf$ can differ from each other by at most $\epsilon^*$ -- we denote this by $||\theta-\theta'||_2\leq \epsilon^*$. Based on this relaxation we define the notion of \textit{membership inference equivalence} with respect to an \MI~attack as follows. Recall, that the output of an attack is a bit, $\calA(x,\m,\psi)=b$, indicating the adversary's belief on whether the point belongs to the dataset.
\begin{defn}[Membership Inference Equivalence] For a given distribution $\D$, two models $\{\m,\mf\}$ are defined to be functionally equivalent w.r.t to an \MI~attack (adversary), $\calA$ iff, \begin{gather}\forall x \in \D, \calA(x,\m,\psi)=\calA(x,\mf,\psi).\end{gather}\label{def:MI:equivalence}\end{defn}
From the above definition, two models are membership inference equivalent equivalent w.r.t an \MI~attack if the \MI~attack's prediction remains the \textit{same} on both the models for all the points in the data distribution $\D$. In other words, the two models are equivalent in terms of the functionality of the attack. We argue that this relaxation is justified because in practice, a membership inferencing is defined by the attack functionality. One can consider this relaxation to be analogous to the assumption of computationally bounded adversaries in cryptography since in practice, adversaries are always computationally bounded. Similar assumptions have also been been made in differential privacy -- Chaudhuri et al.~\cite{Chaudhuri19} proposed capacity-bounded differential privacy where the
adversary that distinguishes the output distributions is assumed to be  bounded in terms of the function
class from which their attack algorithm is drawn. We refer to $\mf$ as the forged model in the rest of the paper. Based on this, we propose the following conjecture.

\begin{conj} For some small value of $\epsilon'$, if $||\theta-\theta'||_2\leq \epsilon'$, then the two models $\{\m,\mf\}$ are membership equivalent w.r.t an \MI~attack. \label{conj:1}\end{conj}
The above conjecture states that if two models are sufficiently close to each other in the parameter space, then they are functionally equivalent w.r.t \MI~attacks. The intuitive reasoning behind this is that models that are very close to each other in the parameter space result in very similar scores for an \MI~attack leading to the same prediction.

\par 
Based on the above relaxation, the \PoR~now should be able to prove that it is computationally feasible for the model owner to have generated \textit{almost the same} model, $\mf$, from a different dataset $D'$. Now recall that a \PoL~(Section \ref{sec:background:PoL}) for any model $\m$ trained on any dataset $D$ logs its complete training trajectory, (i.e., sequence of SGD model updates). Given the \PoL, one can replicate the entire training trajectory starting from the initialization state to the final model weights, $\theta$, and verify its correctness  (up to some tolerance parameter $\epsilon$). 
Reproducing the alleged computation is synonymous to
showing its plausibility. This clearly establishes the \textit{computational feasibility} of obtaining $\m$ from the dataset $D$.  Hence, for the rest of paper we use \textsf{PoL}s to establish computational feasibility.
\par Now note that the model owner (challenger) who has trained $\m$ on the dataset $D$ automatically has access to the \PoL~for the original dataset $D$.  Based on our above discussion, the construction of a \PoR~for data point $\x$ reduces to generating a \PoL~for the model $\mf$ for the dataset $D'= D \setminus \{\x\}$. This is exactly what is enabled by forgeability (Section \ref{sec:background}). Specifically, forgeability allows the model owner to create a forged \PoL~for $\langle \mf, D'  \rangle$ from the original \PoL~for $\langle \m, D\rangle$ such that $||\theta-\theta'||_2\leq \epsilon'$ for some tolerance parameter $\epsilon'$. 
Thus, the model owner can re-purpose a forged \PoL~for $\langle \mf, D\setminus \{\x\}\rangle$ as a \PoR~for refuting a membership inference claim on $\x$. This completes our construction of a valid \PoR.

\subsection{Algorithm for PoR~Generation}\label{sec:PoR:algorithm}
In this section, we describe the algorithm to construct {\PoR}s~in practice. Let $D=\{x_1,\cdots,x_n\}$ be the set of training data. Let the original model be trained with $\T$ iterations, where $B^{(t)}_*$ is the mini-batch used for the $t$-th iteration, $t\in [\T]$. The goal is to generate $n$ {\PoR}s~(i.e., forged {\PoL}s) from $D'_{-i}=D\setminus \{x_i\}$ for all $i\in [n]$. Recall that a forged \PoL~from $D'_{-i}$ is a sequence of mini-batches $\{B^{(t)}_\texttt{forge}(x_i)\}$ such that

\begin{itemize}
\item $x_i\notin B^{(t)}_\texttt{forge}(x_i)$, and 
\item For each iteration $t$, the gradient of the model using $B^{(t)}_\texttt{forge}(x_i)$ closely approximates the gradient computed from the original mini-batch $B^{(t)}_*$.
\end{itemize} 

The naive strategy is to run Algorithm \ref{alg: forging attack} \cite{Thudi22} iteratively and generate $n$ forged \textsf{PoL}s. However, it has the following limitations:

\begin{itemize}
\item \textbf{Computation cost.} It needs to compute $n\cdot\M\cdot\T$ gradients in total (Step 3-9 in Algorithm \ref{alg: forging attack}), where $\M$ is the number of candidate mini-batches (see Section \ref{sec:background:forgeability}). This is computationally expensive when $n$ is large.
\item \textbf{Limited applicability.} The proposed algorithm works only for the simple update rules and cannot accommodate more complex training techniques such as data augmentation. 
\end{itemize} 

To address these challenges, we propose an improved two-phase mechanism to generate the {\PoR}s. In the first phase, we compute and store the forged mini-batches $B^{(t)}_\texttt{forge}(x_i)$ for all data points $i\in [n]$ and  iterations $t\in [\T]$ using . We propose an efficient algorithm (Algorithm \ref{alg: proposed}) to do this, and it is able to handle more complex training techniques. 

The second phase is to reconstruct the \PoR~from stored information \textit{on-the-fly}. This step is done when the model owner is challenged with a membership inference claim from an attacker post-deployment. 
With the two-step mechanism we can avoid saving the entire {\PoR}s (including all details, such as model parameters for every iteration) during the training phase, which requires prohibitively large storage costs and I/O burden especially for modern, big \ML~models. 

Our mechanism is based on three key ideas as follows:\\

\textbf{Key Idea 1: Re-using Forged Mini-batches.}
To reduce the overhead of expensive gradient computation, we propose to \textit{re-use forged mini-batches} across {\PoR}s. For each iteration $t \in [\T]$, we randomly split the training dataset into $\K$ subsets, $D_1^{(t)},\cdots,D_\K^{(t)}$, where $\K$ is a hyper-parameter with a small value (for instance, $\K=5$) independent of $n$. 
%
These subsets are equi-sized and disjoint: 
\begin{equation}\label{eq: rand split condition 1}
    |D_k^{(t)}| = n/\K; D_k^{(t)} \cap D_{k'}^{(t)} = \emptyset \text{ when } 1\leq k\neq k'\leq \K.
\end{equation}
For the $t$-th iteration, we randomly sample $\M$ candidate mini-batches from $D\setminus D_k^{(t)}$ and store the one that best approximates the original gradient. For all $x_i\in D_k^{(t)}$, we let choose this as the forged mini-batch $B^{(t)}_\texttt{forge}(x_i)$. In this way, the number of gradient computations is reduced to $\K\cdot\M\cdot\T$.

\textbf{Key Idea 2: Combining {\PoR}s.}  
The second step in our mechanism is to reconstruct all $n$ {\PoR}s and thus requires $n\cdot\T$ gradient computations. In order to further reduce the computation overhead, we partition the training data into groups of size $\R$. For each group $\hat{X}_i, i \in [n/\R]$, we compute the \PoR~from $D\setminus \hat{X}_i$. The computed \PoR~can thus work for \textit{all} $\R$ data points in $\hat{X}_i$. To achieve this, we require for any $ i \in [n/\R]$ at any step $t\in [\T]$, the $\R$ samples 
\begin{equation}\label{eq: rand split condition 2}
    \hat{X}_i \subset D_k^{(t)}.
\end{equation}

for some $k \in [\kappa]$. This reduces the number of \textsf{PoR}s to be generated to $n/\R$, and therefore reduces the number of gradient computes in the reconstruction phase to $n\cdot\T/\R$.

\textbf{Key Idea 3: Extension to Practical Techniques.} 
Data augmentation is frequently used in practice during model training for improving the model accuracy and generalization. This constitutes augmenting with original training dataset with data points with random transformations. To extend our algorithm to models trained with data augmentation, we let training mini-batches $B_*^{(t)}$ to contain not only the indices but also information about the specific transformations applied. After selecting the $\M$ candidate mini-batches in our algorithm, we apply the same data augmentation to these mini-batches and proceed. Specifically, we consider the common \textit{random horizontal flip} augmentation, which randomly flips an image with $0.5$ probability. We store a binary flag along with each sample index to indicate whether the corresponding sample is flipped. 

In addition to data augmentation, certain modifications to the SGD update rule can also boost model performance in practice. Three common techniques are: momentum, weight decay, and learning rate scheduling. We extend our algorithm to support these techniques as follows. $(1)$ Momentum in SGD means the update rule not only considers the current gradients but also those in the previous step. We do not modify the way to generate the \PoR~because we assume the gradients of forged mini-batches in the previous step are close to the true gradients in the previous step. We use the new update rule when reconstructing forged models. $(2)$ Weight decay means there is an $\ell_2$ regularization term of the model parameters in the loss function. We use the new loss function to compute gradients of the original and forged mini-batches. $(3)$ Learning rate scheduling means the learning rate $\mathrm{step\_size}$ changes during training. We use the corresponding learning rates when reconstructing the forged models. 


\textbf{Summary.} The construction of our mechanism is outlined in Algorithm \ref{alg: proposed}, where \highlight{orange lines} refer to modifications we make to the original Algorithm \ref{alg: forging attack}. The full algorithm that includes all practical techniques is presented in Algorithm \ref{alg: complete} in Appendix \ref{app:algorithm}.
The total number of gradient computation in our algorithm is $\K\cdot\M\cdot\T+n\cdot\T/\R$. It is much more efficient than directly applying Algorithm \ref{alg: forging attack} to all training samples, which requires $n\cdot\M\cdot\T$ gradient computes. \footnote{For example, by taking $n=5K, \M=200, \K=5, \R=10$, our algorithm is $1667\times$ faster.}

\textbf{Discussion.} Here we discuss how the proposed \PoR~construction meets the desiderata listed in Section \ref{sec:PoR:def}. At the core of a \PoR~is a log of the training trajectory which can be easily replicated and verified (upto a tolerance parameter) which satisfies \textbf{P1}. \textbf{P2} is met since the proposed construction of the \PoR~is completely attack agnostic. Given a target data point $\x$, the same \PoR~works for \textit{any} \MI~attack -- even one with very a high confidence or low false positive rate (see Section \ref{sec:practice}). The \PoR~generation introduces no modifications \footnote{A \PoR~can be generated on-the-fly as long as the model owner maintains a log of the gradient descent update trajectory during training. We propose Algorithm \ref{alg: proposed} \ref{alg: proposed} for pre-computing the forged mini-batches only to reduce computational cost of generating the \textsf{
PoR}s for all the data points in $D$.} to the training pipeline thereby satisfying \textbf{P3}. A model owner who trains the model would have access to its training trajectory by default, enabling them to construct a \PoR~efficiently.  Additionally, our proposed Algorithm \ref{alg: proposed} shows how we can piggyback the \PoR~generation process with the model training for efficient generation of \textsf{PoR}s for all the data points. This satisfies condition \textbf{P4}. For  \textbf{P5},  instead of recomputing the gradient for \textit{every} update step recorded in the log, the verifier can check the correctness of only a \textit{subset} of the update steps. This allows the verifier to trade-off the confidence of verification with its computational cost -- greater the number of  updates steps checked, higher is the confidence in the proof's validity. As suggested in \cite{Jia21}, a heuristic could be to check the validity of only the largest model updates. The tolerance parameter can be set up to account for the errors incurred due to the usage of different hardwares. 

\begin{algorithm*}
    \centering
    \caption{Generation of PoRs (simplified; see complete algorithm in Appendix \ref{app:algorithm})}
    \label{alg: proposed}
    \begin{algorithmic}[1]
        \State \textbf{Inputs:} $D=\{x_1,\cdots,x_n\}$ - Training dataset;  $\T$ - Total number of iterations; $B_*^{(t)}$ - Training mini-batches  for the $t$-th iteration where $t\in [\T]$; $b$ - Size of each mini-batch, $\left|B_*^{(t)}\right|$; $\theta^{(0)}$ - Initialization parameters.
        \State \textbf{Hyper-parameters:} $\texttt{step\_size}$; $\K$ - Number of splits; $\M$ - Number of candidate mini-batches sampled every iteration for forging.
        \For{$t=1,\cdots,\T$}
            \State Compute gradients and update the original model: \\
            $$\texttt{grad}^{(t)} = \frac1b\nabla_{\theta} \sum_{x\in B_*^{(t)}}\ell(x;\theta^{(t)});~~\theta^{(t)}=\theta^{(t-1)}-\texttt{step\_size}\cdot\texttt{grad}^{(t)}.$$
            \State \highlight{Randomly split $D=D_1^{(t)}\cup\cdots\cup D_\K^{(t)}$ that satisfy Eqs. \eqref{eq: rand split condition 1} and \eqref{eq: rand split condition 2}.}
            \For{$k=1,\cdots,\K$}
                \State Randomly select subset $D_{\mathrm{sub}}$ from $D\setminus D_k^{(t)}$.
                \State \highlight{Randomly select $\M$ mini-batches with size $b$ from $D_{\mathrm{sub}}$: $B_1^{(k,t)},\cdots B_\M^{(k,t)}$.}
                \State Choose the best mini-batch to approximate the gradient: \\
                $$m_*^{(k,t)} = \arg\min_{m \in [\M]} \left\|\frac1b\nabla_{\theta} \sum_{x\in B_m^{(k,t)}}\ell(x;\theta^{(t)})-\texttt{grad}^{(t)}\right\|_2^2.$$
                \highlight{
                    \For{$x_i\in D_k^{(t)}$}
                    \If{$x_i\in B_*^{(t)}$}
                        \State $B_{\texttt{forge}}^{(t)}(x_i) = B_{m_*^{(k,t)}}^{(k,t)}$;
                    \Else 
                        \State $B_{\texttt{forge}}^{(t)}(x_i) = B_*^{(t)}$.
                    \EndIf
                    \EndFor 
                }
            \EndFor
        \EndFor
        \State \textbf{Return:} original model $\theta_*=\theta^{(\T)}$; all forged mini-batches $\{B_{\texttt{forge}}^{(t)}(x_i):  i\in [n], t \in [\T]\}$.
    \end{algorithmic}
\end{algorithm*}

\subsection{Can membership inferencing be always refuted?}\label{sec:PoR:impossibility}

It is important to note that \textsf{PoR}s have limitations in practice -- we \textit{cannot} create a valid \PoR~for \textit{all} possible data points. 
We provide a formal impossibility result for a logistic regression model as follows.
\begin{theorem}
Let $f(x) = \mathrm{sigmoid}(w^{\top}x)$ with weight vector $w$. Let $\mathrm{loss}(\cdot)$ be the cross entropy loss in Logistic regression. Assume $x_2,\cdots,x_n$ all lie on a subspace $\Gamma$ and $x_1\notin \Gamma$. Let the batchsize of SGD be 1 and  $y_1\neq \mathrm{sigmoid}(w^{\top}x^*)$. Under the above condition, it is impossible to generate a valid \PoR~using Algorithm \ref{alg: proposed}.\label{thm:impossibility}
\end{theorem}
\begin{proof}
The gradient at $x_i$ is 
\[\nabla_w \mathrm{loss}(x_i) = \left(\mathrm{sigmoid}(w^{\top}x_i)-y_i\right)x_i.\]
Then, 
\begin{equation}
    \begin{array}{l}
        \displaystyle \min_{2\leq i\leq n} \left\|\nabla_w \mathrm{loss}(x_1) - \nabla_w \mathrm{loss}(x_i)\right\| \\ 
        \displaystyle \geq \min_{c\in\mathbb{R},x\in \Gamma} \left\|cx - \left(\mathrm{sigmoid}(w^{\top}x_1)-y_1\right)x_1\right\| \\
        \displaystyle \geq \left|\mathrm{sigmoid}(w^{\top}x_1)-y_1\right| \mathbf{dist}(x_1,\Gamma) > 0.
    \end{array}
\end{equation}
Thus, clearly no sample from $D\setminus\{x_1\}$ can produce the same gradient as $x^*$, thus violating the assumption of Lemma 4 in \cite{Thudi22} which is key for generating a \PoL.
\end{proof}
Another instance where one cannot refute a membership inferencing claim is the extreme case when the adversary knows all but one data point in $D$. 
 We also cannot hope to create a valid \PoR~for an outlier/ out-of-distribution data point $\x$. It is because, recall that the key idea of generating the \PoR~(i.e., forged \PoL) is to replace a mini-batch $B$ involving $\x$ with another $B_{\texttt{forge}}$ such that $\x \not \in B_{\texttt{forge}}$ and the gradients of both are (almost) the same.  However, for outliers it might not be possible to find appropriate alternative mini-batches in the dataset $D\setminus \{\x\}$ such that the resulting gradients are sufficiently close. 

\section{Can we construct valid PoRs in practice?}\label{sec:practice}

In this section, we empirically evaluate whether valid \textsf{PoR}s can be constructed in practice. Recall that the core technique behind the construction of a \PoR~is forging a \PoL.

To this end, we study the following three questions:
\begin{itemize}
    \item \textbf{Q1.} What is the quality of the forged \textsf{PoL}s?
    \item \textbf{Q2.} What is the impact of the different algorithmic hyperparameters on the quality of forging?
    \item \textbf{Q3.} Do the forged mini-batches have the same distribution of samples as that of the original ones?
\end{itemize}

\subsection{Experimental Setup}\label{sec:practice:setup}

\textbf{Datasets and Models.}
We look at two image datasets, \textit{MNIST}~\cite{MNIST} and \textit{CIFAR-10}~\cite{CIFAR10}.

For \textit{MNIST}~\cite{MNIST}, we train a \textit{LeNet5}~\cite{lecun1989backpropagation} model with 61.7K parameters. We use $\texttt{step\_size}=1\times10^{-2}$, batch size $b=100$ and $20$ training epochs. 
The default values of the number of candidate mini-batches is $\M=200$ and number of random splits is $\K=5$.

For \textit{CIFAR-10}~\cite{CIFAR10}, we look at two models: \textit{ResNet-mini} \cite{He_2016_CVPR} with 1.49M parameters, and \textit{VGG-mini} \cite{simonyan2014very} with 5.75M parameters. We set $\texttt{step\_size}=1\times10^{-2}$, batch size $b=100$, and number of training epochs $=20$. All models are augmented  with random horizontal flips data augmentation, where each data point is randomly flipped with probability $0.5$ each time. In addition to the standard SGD update rule, we study a modified SGD setting as introduced in Section \ref{sec:PoR:algorithm}. We use notation $\dag$ to represent this setting. Specifically, we set momentum to be $0.9$, weight decay coefficient to be $5\times10^{-4}$, and use cosine anneal for the learning rate scheduling (see Appendix \ref{app:eval} for details). We set the number of candidate mini-batches as $\M=200$ and number of random splits as $\K=2$.

\textbf{Experiment Design.} 
Recall that the construction of a valid \PoR~constitutes 
\begin{itemize}
\item \textbf{Condition 1.} Forging a \PoL~for $\langle \mf, D' \rangle$ where $||\theta-\theta'||_2< \epsilon'$ such that $\epsilon'$ is reasonably low.
\item \textbf{Condition 2.} Ensuring that the models $\m$ and $\mf$ satisfy membership inference equivalence.
\end{itemize}

The goal of the experiments in this section is to generate a \PoR~for each point in $D$ and check how many of them are actually valid, i.e., satisfy the aforementioned criterion. The naive mechanism is to generate $n$ distinct \textsf{PoR}s, one for each data point. However, the resulting overhead would be very high and is beyond the computational powers of our academic resources. Hence, we explore two  alternatives that approximates this ideal: 
\begin{itemize}
\item {\textbf{Sampling Approximation.}} Instead of generating a \PoR~for all data points, we do it for a subset of data points randomly selected from the training dataset $D$ .
\item {\textbf{Combination Approximation.}} As discussed in Section \ref{sec:PoR:algorithm} (Key Idea 2), we re-use the same \PoR~for $\R$ data points at a time for the entire training dataset. 
\end{itemize}

In the sampling approximation setting, for both MNIST and CIFAR-10, we use a random subset of $n=10$K data points with $\T=2$K and $\R=1$.

In the combination approximation setting, we use the entire training dataset of size $n=60$K  with $\T=12$K and $\R \in \{10,100\}$ for MNIST. For CIFAR-10, we use the entire training dataset of size $n=50$K with $T=10$K and $\R=10$. More details are presented in Appendix \ref{app:eval}.

\textbf{\MI~attacks.} We evaluate the following \MI~attacks:

\begin{itemize}
\item \textit{EnhancedMIA}~\cite{EnhancedMIA} is currently the state-of-the-art \MI~attack which uses a likelihood ratio test approximately comparing the loss of the target with a threshold. 
\item \textit{LiRA}~\cite{carlini2021membership} is one of the most recent \MI~attacks with a special focus on low false positive rate for the \MI~predictions. It performs a likelihood ratio test for the loss of $x^*$ between the distributions of losses from shadow models trained with and without $x^*$.
\item \textit{Xent}~\cite{Yeom18} is an entropy-based \MI~attack, which computes the cross entropy of the target model output at $x^*$, and then compares to a label-specific threshold derived from a shadow model. 
\item \textit{MEntr}~\cite{MEntr} is an improved \MI~attack of Xent, which takes into consideration of correct and wrong classification results when computing the entropy of the target model output at $x^*$. 
\item \textit{MIDA}~\cite{yu2021does} is an \MI~attack designed for models trained with data augmentation. It classifies sets of augmented samples instead of a single sample by extracting features from losses of the target model output at $x^*$ after data augmentation. Since we use data augmentation only for CIFAR-10, we evaluate MIDA only on this dataset.
\end{itemize}

The training details including shadow models and threshold selection methods are presented in Appendix \ref{app:eval}. 

\textbf{Metrics.} We use the following three metrics for our evaluation. 

To evaluate in Condition 1, we use the \textit{$L_2$-distance} metric precisely defined as follows. Let $M_{\theta_{-i}}$ represent the forged model corresponding to the target of the \PoR~$\x=x_i, i \in [n]$. The first metric is the $\ell_2$ distance between the parameters of the original model, $\theta_*$, and forged model, $\theta_{-i}$. We report the mean value for $\d=||\theta_*-\theta_{-i}||_2^2/\mathrm{dim}(\theta_*)$. 

To evaluate the second condition, we look at the following two metrics. First, is the \textit{membership prediction difference} metric. Let $U$ be the set of data points we want to carry out the \MI~attack on. Recall that $M_{\theta_{-i}}$ is the forged model used for constructing a \PoR~for the data point $x_i, i \in [n]$. The metric counts the number of data points in the training dataset $D$ for which the \MI~attack results in different predictions for at least one of the samples in $U$ for the original and the forged models:  $\c=\text{Percentage of data points satisfying }|\{x_i \in D| \exists z \in U \text{ s. t. } \calA(z,M_{\theta_*},\psi)\neq\calA(z,M_{\theta_{-i}},\psi)\}|$. To better understand where the \MI~predictions differ, for each pair of models $\{M_{\theta_*}, M_{\theta_{-i}}\}$  we consider the following three cases separately: $(1)$ In the first case, we look at the \MI~attack's predictions for the the target data point. Note that $x_i$ is a member of original model $M_{\theta_*}$ but a non-member of the forged model, $M_{\theta_{-i}}$. We refer to it as the \texttt{Diff} setting. $(2)$ The second case corresponds to data points that are \textit{common} to both the models, i.e., $D_{-i}=D\setminus \{x_i\}$. We refer to this as the \texttt{Common} setting. Due to our constraint of academic resources, we perform the following approximation. Instead of testing on the entire $n-1$ data points in  $D_{-i}$, we choose a subset of $5$ data points sampled at random for each $i\in [n]$.    $(3)$ In the third case, we carry the \MI~attacks on data points of the validation dataset which are non-members for both the models. Due to our constraint of academic resources, we again carry out the test on a subset of $5$ sampled at random for each $i\in [n]$. We refer to this as the \texttt{Validation} setting. More details are presented in Appendix \ref{app:eval}.

The final metric is the \textit{membership score difference} metric which measures the change in the prediction scores of the \MI~attack between the forged and original models: $\s=|\calA.\texttt{score}(x_i,M_{\theta_{-i}}, \psi)-\calA.\texttt{score}(x_i,M_{\theta_*}, \psi)|$. 


\subsection{Quality of Forged PoLs} \label{sec:practice:results}

In this section, we answer Q1. For this, we look at the two conditions for generating a valid \PoR~as discussed above.

For evaluating Condition 1, we report the $L_2$-distance metric $\d$ in Fig. \ref{fig: model dist}. We observe that the forged models are very close to the original models for all of our experimental settings. 
Additionally, we observe that the distance increases for $\R=10$. This is because it is now more difficult to find forged mini-batches that approximate the original gradients well. Thus, the parameter $\R$ trades-off the computational overhead of generating a \PoR~with its quality.

Next, we discuss our evaluation of Condition 2. We start by analyzing the membership prediction difference metric in the \texttt{Diff} setting (Table \ref{tab: PoR failure rate setting 1}). If the \MI~attack were perfect, then the metric would have a value of $100\%$. On the other hand, if forging were perfect, then the metric would have a value of $0\%$.  Our empirical results show that the forging  works quite well - the \MI~predictions differ only for at most $0.3\%$ and $12.4\%$ (Xent attack on VGG-mini with $n=50K^{\dagger}$, $\R=10$ and $\K=2$) of all the data points in MNIST and CIFAR-10, respectively.
\par For the \texttt{Common} setting (Table \ref{tab: PoR failure rate setting 2}), if the \MI~attack were perfect, the metric would have a value of $0\%$. This is because the data points in $D_{-i}$ are members of \textit{both} the models $\{M_{\theta_*}, M_{\theta_{-i}}\}$.  Forging does not affect the data points in $D_{-i}$ and hence, ideally we expect the metric to be $0\%$. Here, a disagreement on the prediction is an indication of a false negative error on \textit{one} of the models. We observe that the \MI~attacks have different predictions for at most $1.2\%$  and $11.08\%$ of the data points for MNIST and CIFAR-10, respectively.  It is important to note that the LiRA and EnhancedMIA attacks have a much lower prediction difference -- 0.011\% and 0.35\% for MNIST and CIFAR-10, respectively. This is because they are currently the state-of-the-art attacks that are designed to maximize prediction accuracy under a \textit{low} false positive rate to improve the reliability of \MI~attacks (see our discussion in Section \ref{sec:practice:discussion}). Corollarily, the quality of the forging is significantly better under these state-of-the-art attacks.

Finally for the \texttt{Validation} setting (Table \ref{tab: PoR failure rate setting 3}), the metric would again have a value of $0\%$ under a perfect \MI~attack. This is because the data points in the validation dataset are \textit{not} members for both the models. Additionally, since forging is completely agnostic of these data points, we again expect a value of $0\%$ ideally. A difference in the two predictions here means that the \MI~attack has a false positive error on one of the models.  We observe that \MI~attacks have different predictions for $0.347\%$  and $11.3\%$ of the data points for MNIST and CIFAR-10, respectively.  Similar to above, the prediction difference is lower (consequently, the quality of forging is better) for LiRA and EnhancedMIA attacks -- $0.111\%$ and $6.336\%$ for MNIST and CIFAR-10, respectively.



\begin{figure}[!t]
    \centering
    
    \begin{subfigure}[t!]{0.45\textwidth}
    \includegraphics[width=0.99\textwidth]{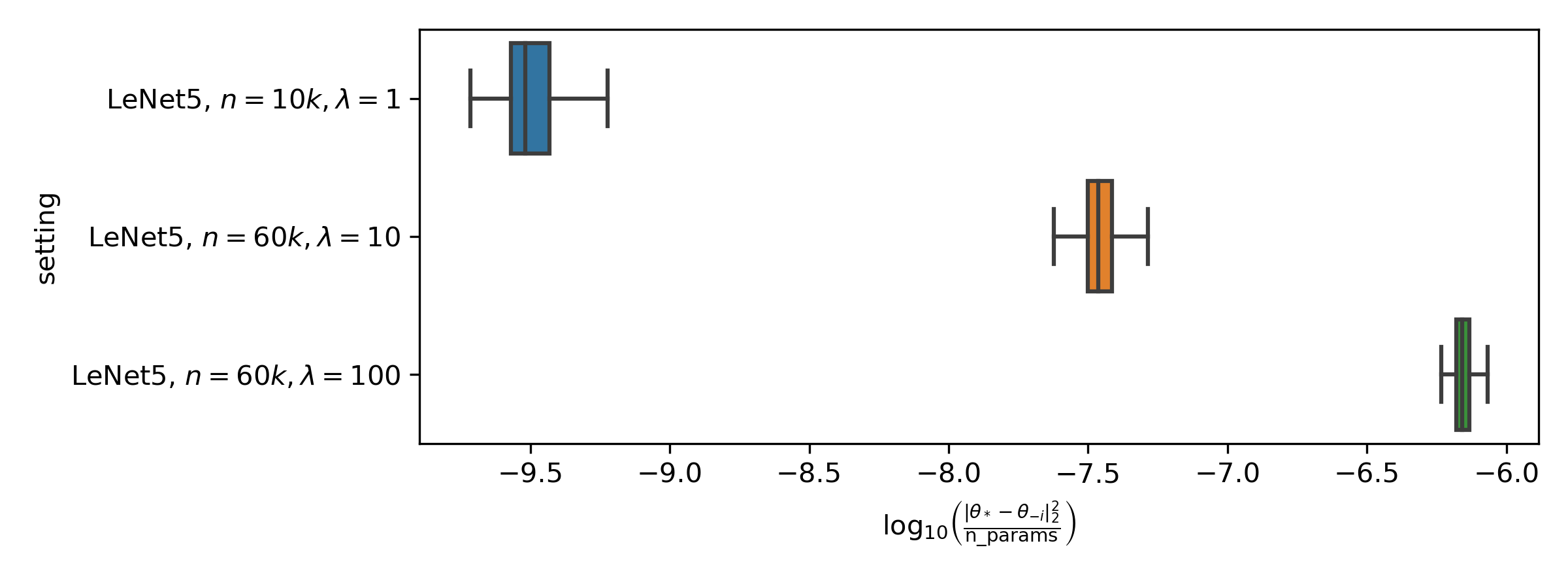}
    \caption{MNIST.}
    \label{fig: model dist mnist}
    \end{subfigure}\\
    \vspace{1em}
    \begin{subfigure}[t!]{0.45\textwidth}
    \includegraphics[width=0.99\textwidth]{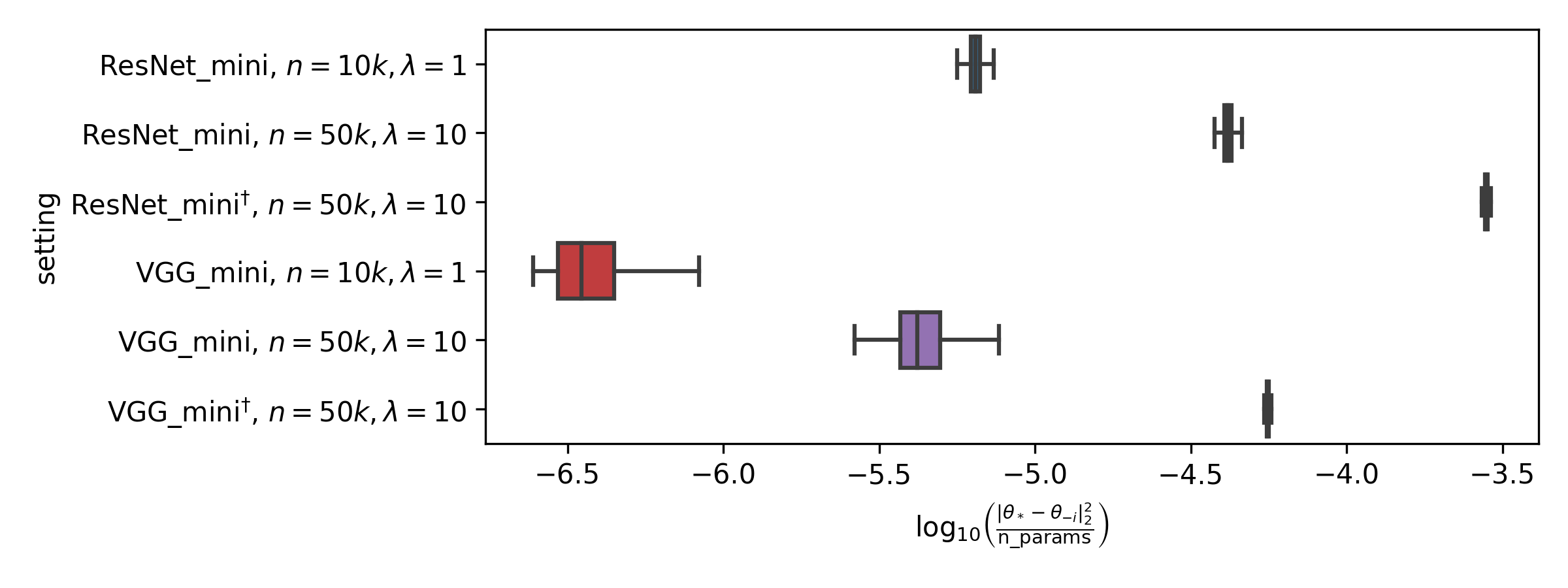}
    \caption{CIFAR-10 ($\dag$ means the modified SGD is used)}
    \end{subfigure}
    
    \caption{Box plots of $L_2$ model distances ($d_{\theta}$) in the $\log$ scale: $\log_{10}(||\theta_*-\theta_{-i}||_2^2/\mathrm{dim}(\theta_*))$. Smaller values indicate the forged models are closer to the original model.}
    \label{fig: model dist}
\end{figure}

In addition to the difference in the prediction bits, we also evaluate the membership score difference metric,  $\s$, for completeness. For MNIST (Fig. \ref{fig: MI abs diff mnist}), $\s$ is almost close to zero for all the \MI~attacks.  For CIFAR-10 (Fig. \ref{fig: MI abs diff mnist}), $\s$ is close to zero for most of the data points but with a slightly higher variance. This is expected since the models for CIFAR-10 are much larger and more complex than MNIST. Additional results are presented in Appendix \ref{app:eval}.


\begin{figure*}[!t]
    \centering
    
    \begin{subfigure}[t!]{0.38\textwidth}
    \centering
    \includegraphics[trim=25 25 0 0, clip, height=0.95in]{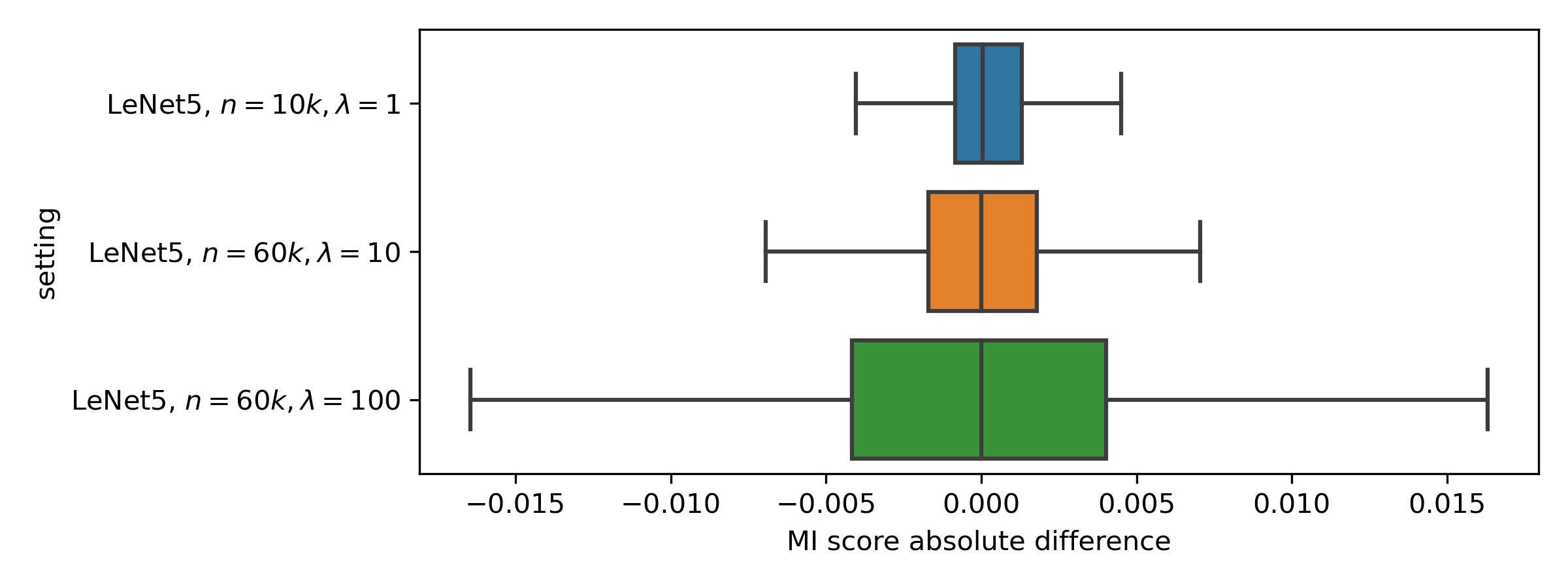}
    \caption{LiRA}
    \end{subfigure}
    \begin{subfigure}[t!]{0.3\textwidth}
    \centering
    \includegraphics[trim=150 25 0 0, clip, height=0.95in]{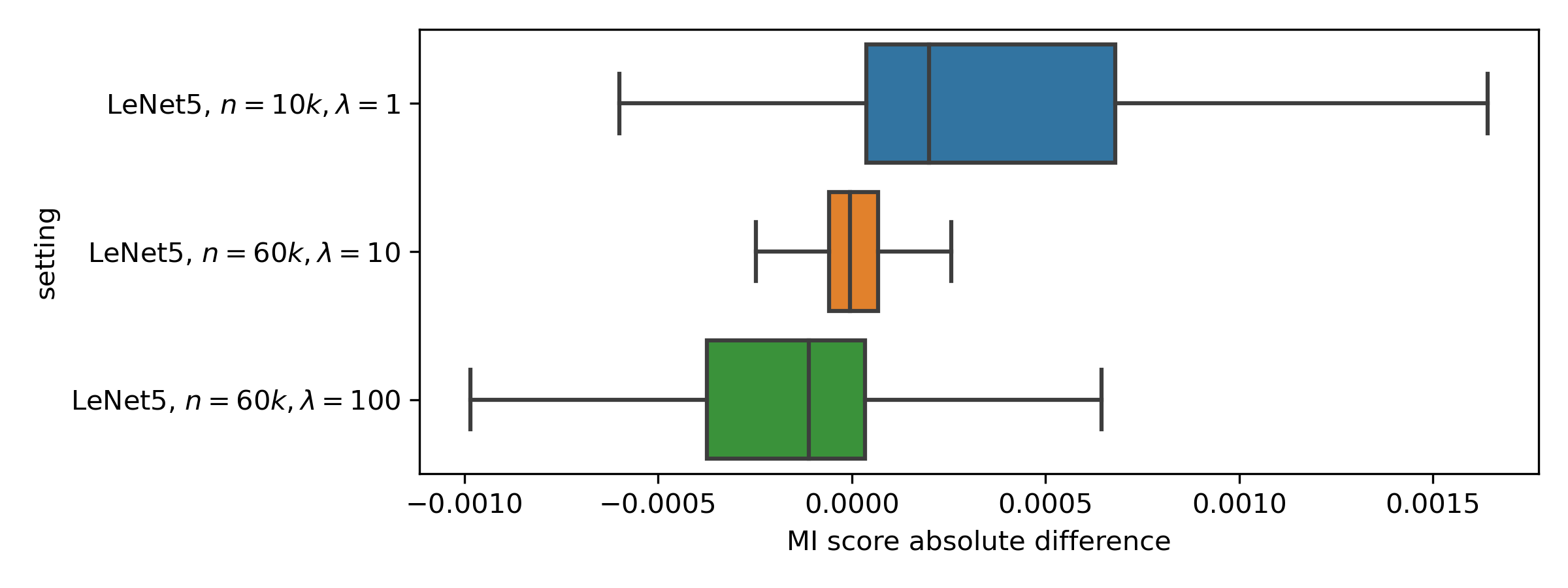}
    \caption{EnhancedMIA}
    \end{subfigure}
    %
    %
    \begin{subfigure}[t!]{0.3\textwidth}
    \centering
    \includegraphics[trim=150 25 0 0, clip, height=0.95in]{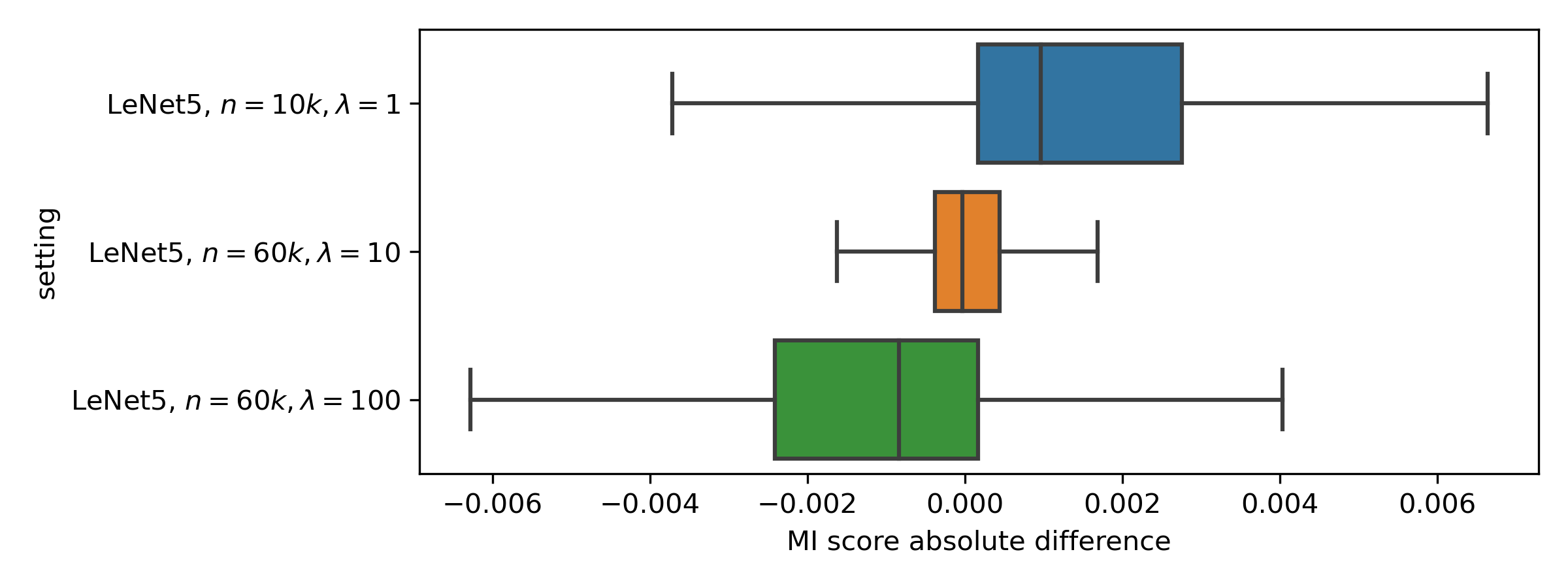}
    \caption{Xent}
    \end{subfigure}
    
    \caption{Box plots of membership score differences ($\s$) on MNIST. Most score differences are close to zero, indicating these MI attacks output similar scores between original and forged models.}
    \label{fig: MI abs diff mnist}
\end{figure*}

\begin{figure*}[!t]
    \centering
    
    \begin{subfigure}[t!]{0.38\textwidth}
    \centering
    \includegraphics[trim=25 25 0 0, clip, height=1.3in]{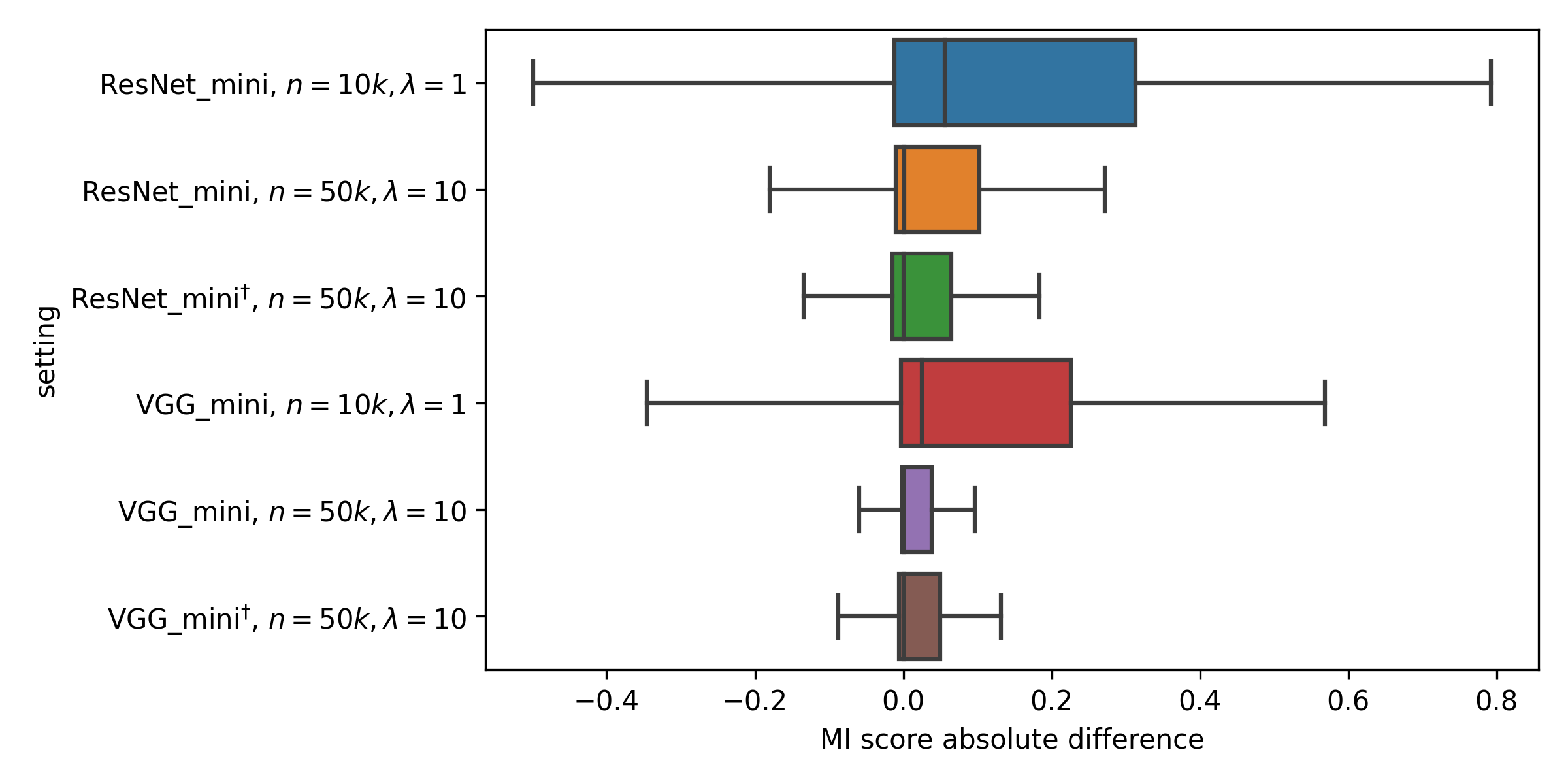}
    \caption{LiRA}
    \end{subfigure}
    \begin{subfigure}[t!]{0.3\textwidth}
    \centering
    \includegraphics[trim=175 25 0 0, clip, height=1.3in]{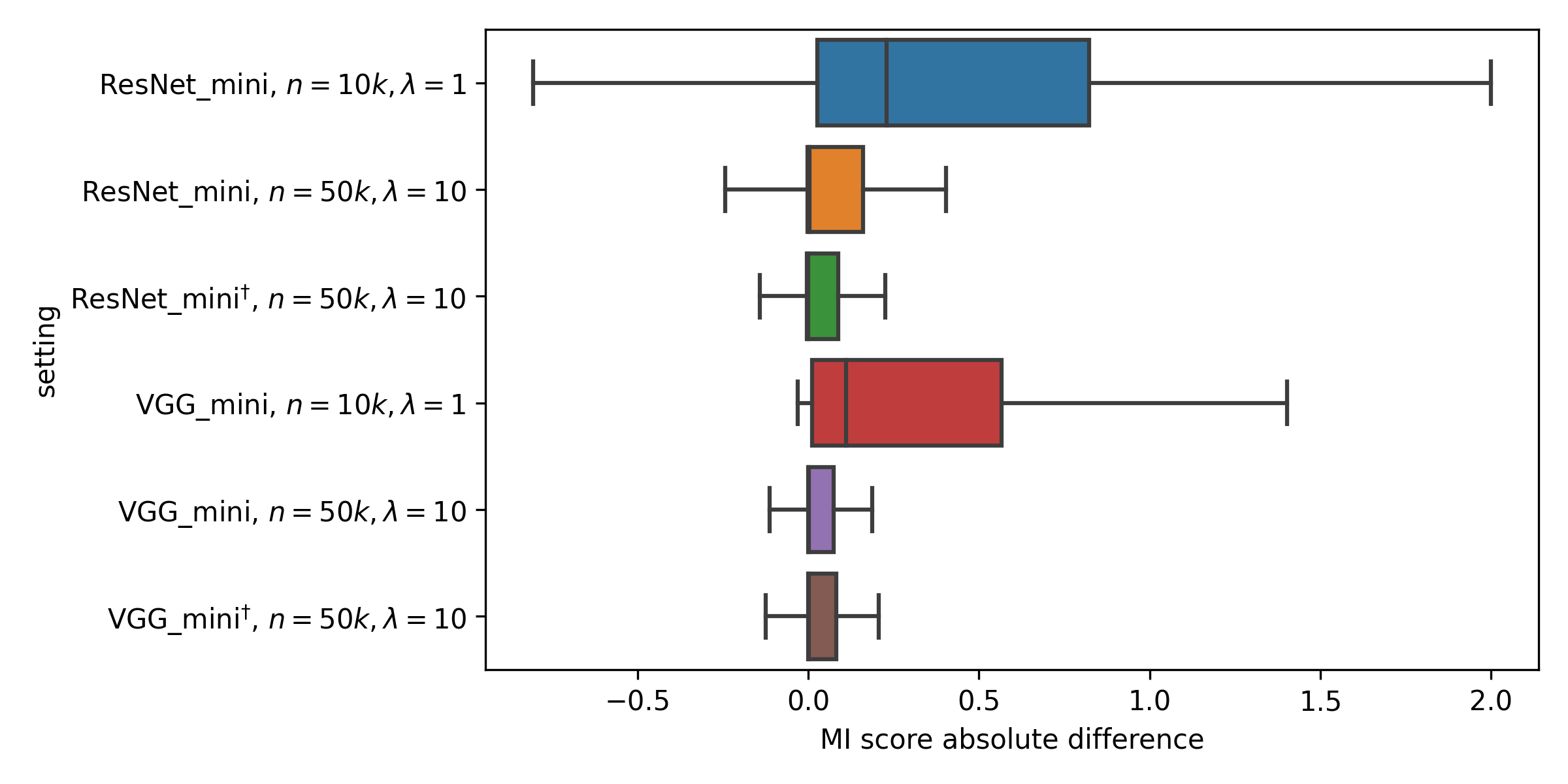}
    \caption{EnhancedMIA}
    \end{subfigure}
    %
    %
    \begin{subfigure}[t!]{0.3\textwidth}
    \centering
    \includegraphics[trim=175 25 0 0, clip, height=1.3in]{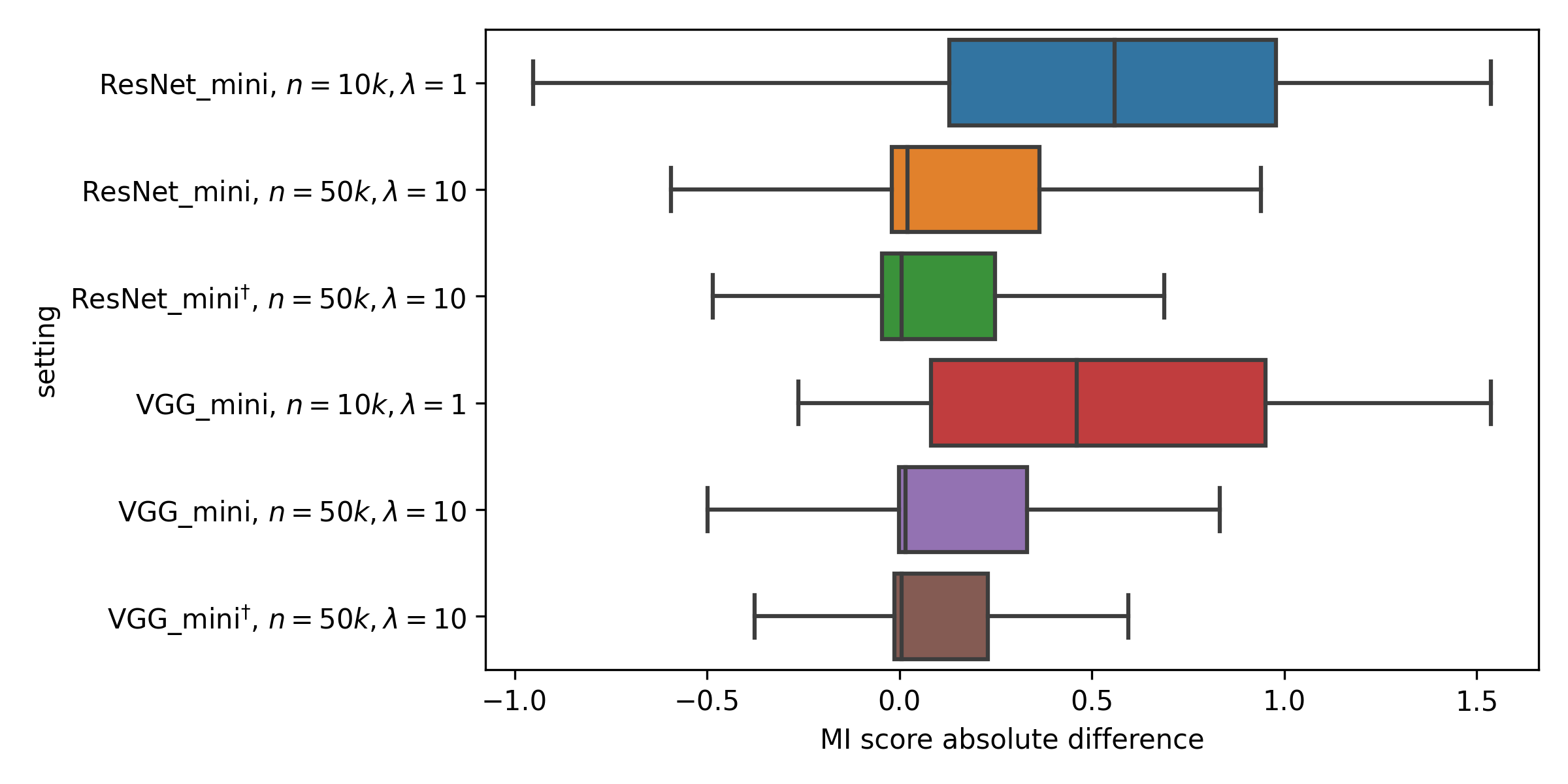}
    \caption{Xent}
    \end{subfigure}
    %
    
    \caption{Box plots of membership score differences ($\s$) on CIFAR10 ($\dag$ means the modified SGD is used). Most score differences are close to zero, indicating these MI~attacks output similar scores between original and forged models. The values and variances are higher than those in MNIST experiments, which is likely because models for CIFAR10 are larger and more complex. There is no significant difference between ResNet and VGG networks.}
    \label{fig: MI abs diff cifar10}
\end{figure*}

\begin{table*}[!t]
    \centering
    \begin{tabular}{cccccccccc}
    \toprule
        \multirow{2}{*}{Dataset} & \multirow{2}{*}{Model} & \multirow{2}{*}{$n$} & \multirow{2}{*}{$\R$} & \multirow{2}{*}{$\K$} & \multicolumn{5}{c}{\% Data points with different \MI~predictions in the \texttt{Diff} setting} \\ \cline{6-10}
        &&&&& LiRA~\cite{Carlini22} & EnhancedMIA~\cite{ye2021enhanced} & MEntr~\cite{MEntr} & Xent~\cite{Yeom18} & MIDA~\cite{yu2021does}
        \\ \hline
        \multirow{3}{*}{MNIST} & \multirow{3}{*}{LeNet5} 
        &  10K & 1   & 5 & 0.00\% & 0.01\% & 0.14\% & 0.25\% & - \\
        && 60K & 10  & 5 & 0.03\% & 0.09\% & 0.32\% & 0.31\% & - \\
        && 60K & 100 & 5 & 0.08\% & 0.57\% & 1.31\% & 1.30\% & - \\ \hline
        \multirow{6}{*}{CIFAR-10} & \multirow{3}{*}{ResNet-mini} 
        &  10K          &  1 & 2 & 0.63\% & 5.46\% & 2.50\% & 5.26\% & 2.39\% 
        \\
        && 50K          & 10 & 2 & 0.46\% & 2.16\% & 3.20\% & 4.69\% & 2.97\% 
        \\
        && 50K$^{\dag}$ & 10 & 2 & 0.42\% & 4.30\% & 8.42\% & 8.51\% & 6.05\% 
        \\ \cline{2-10}
        & \multirow{3}{*}{VGG-mini} 
        &  10K          & 1  & 2 & 0.52\% & 4.45\% & 3.71\% & 3.75\% & 1.45\%
        \\
        && 50K          & 10 & 2 & 0.36\% & 3.68\% & 7.10\% & 7.14\% & 7.26\% 
        \\
        && 50K$^{\dag}$ & 10 & 2 & 0.34\% & 3.63\% & 3.53\% & 12.4\% & 7.30\%
        \\
    \bottomrule
    \end{tabular}
    \caption{Membership prediction differences ($\c$) in the \texttt{Diff} setting. $\dag$ means the modified SGD is used.}
    \label{tab: PoR failure rate setting 1}
\end{table*}

\begin{table*}[!t]
    \centering
    \begin{tabular}{cccccccccc}
    \toprule
        \multirow{2}{*}{Dataset} & \multirow{2}{*}{Model} & \multirow{2}{*}{$n$} & \multirow{2}{*}{$\R$} & \multirow{2}{*}{$\K$} & \multicolumn{5}{c}{\% Data points with different \MI~predictions in the \texttt{Common} setting} \\ \cline{6-10}
        &&&&& LiRA~\cite{Carlini22} & EnhancedMIA~\cite{ye2021enhanced} & MEntr~\cite{MEntr} & Xent~\cite{Yeom18} & MIDA~\cite{yu2021does}
        \\ \hline
        \multirow{3}{*}{MNIST} & \multirow{3}{*}{LeNet5} 
        &  10K & 1   & 5 & 0.00\% & 0.02\% & 0.04\% & 0.06\% & - 
        \\
        && 60K & 10  & 5 & 0.01\% & 0.10\% & 0.30\% & 0.26\% & - 
        \\
        && 60K & 100 & 5 & 0.06\% & 0.57\% & 1.33\% & 1.18\% & - 
        \\ \hline
        \multirow{6}{*}{CIFAR-10} & \multirow{3}{*}{ResNet-mini} 
        &  10K          & 1  & 2 & 0.04\% & 6.29\% & 2.47\% & 2.11\% & 1.04\% 
        \\
        && 50K          & 10 & 2 & 0.23\% & 2.15\% & 3.26\% & 4.85\% & 3.11\% 
        \\
        && 50K$^{\dag}$ & 10 & 2 & 0.35\% & 4.31\% & 8.57\% & 8.62\% & 6.21\% 
        \\ \cline{2-10}
        & \multirow{3}{*}{VGG-mini} 
        &  10K          & 1  & 2 & 0.00\% & 4.69\% & 3.80\% & 3.47\% & 1.05\%  
        \\
        && 50K          & 10 & 2 & 0.07\% & 3.81\% & 7.54\% & 7.59\% & 7.61\% 
        \\
        && 50K$^{\dag}$ & 10 & 2 & 0.22\% & 3.59\% & 3.65\% & 11.1\% & 6.85\%
        \\
    \bottomrule
    \end{tabular}
    \caption{Membership prediction differences ($\c$) in the \texttt{Common} setting. $\dag$ means the modified SGD is used.}
    \label{tab: PoR failure rate setting 2}
\end{table*}

\begin{table*}[!t]
    \centering
    \begin{tabular}{cccccccccc}
    \toprule
        \multirow{2}{*}{Dataset} & \multirow{2}{*}{Model} & \multirow{2}{*}{$n$} & \multirow{2}{*}{$\R$} & \multirow{2}{*}{$\K$} & \multicolumn{5}{c}{\% Data points with different predictions in the \texttt{Validation} setting} \\ \cline{6-10}
        &&&&& LiRA~\cite{Carlini22} & EnhancedMIA~\cite{ye2021enhanced} & MEntr~\cite{MEntr} & Xent~\cite{Yeom18} 
        & MIDA~\cite{yu2021does}
        \\ \hline
        \multirow{3}{*}{MNIST} & \multirow{3}{*}{LeNet5} 
        &  10K & 1   & 5 & 0.00\% & 0.02\% & 0.07\% & 0.06\% & - 
        \\
        && 60K & 10  & 5 & 0.01\% & 0.11\% & 0.35\% & 0.25\% & - 
        \\
        && 60K & 100 & 5 & 0.06\% & 0.57\% & 1.31\% & 1.30\% & - 
        \\ \hline
        \multirow{6}{*}{CIFAR-10} & \multirow{3}{*}{ResNet-mini} 
        &  10K          & 1  & 2 & 0.36\% & 6.34\% & 2.15\% & 5.90\% & 2.71\% 
        \\
        && 50K          & 10 & 2 & 0.34\% & 2.12\% & 3.35\% & 4.63\% & 3.03\% 
        \\
        && 50K$^{\dag}$ & 10 & 2 & 0.31\% & 4.34\% & 7.52\% & 7.64\% & 5.89\%
        \\ \cline{2-10}
        & \multirow{3}{*}{VGG-mini} 
        &  10K          & 1  & 2 & 0.19\% & 4.71\% & 1.97\% & 2.16\% & 1.29\% 
        \\
        && 50K          & 10 & 2 & 0.20\% & 3.77\% & 6.31\% & 6.32\% & 6.61\% 
        \\
        && 50K$^{\dag}$ & 10 & 2 & 0.23\% & 3.58\% & 3.55\% & 11.3\% & 6.95\%
        \\
    \bottomrule
    \end{tabular}
    \caption{Membership prediction differences ($\c$) in the \texttt{Validation} setting. $\dag$ means that modified SGD is used.}
    \label{tab: PoR failure rate setting 3}
\end{table*}

\begin{table}[!t]
    \centering
    \begin{tabular}{cccccc}
    \toprule
        {Dataset} & {Model} & {$n$} & {$\R$} & {$\K$} & $\ell_1$ distance \\ \hline
        \multirow{3}{*}{MNIST} & \multirow{3}{*}{LeNet5} & 10k & 1 & 5 & $0.0077 \pm 0.0001$ \\
        && 60k & 10 & 5 & $0.0249 \pm 0.0002$ \\
        && 60k & 100 & 5 & $0.0941 \pm 0.0005$ \\ \hline
        \multirow{6}{*}{CIFAR-10} & \multirow{3}{*}{ResNet-mini} & 10k & 1 & 2 & $0.0166 \pm 0.0001$ \\
        && 50k & 10 & 2 & $0.0285 \pm 0.0002$ \\
        && 50k$^{\dag}$ & 10 & 2 & $0.0285 \pm 0.0002$ \\ \cline{2-6}
        & \multirow{3}{*}{VGG-mini} & 10k & 1 & 2 & $0.0166 \pm 0.0001$ \\
        && 50k & 10 & 2 & $0.0285 \pm 0.0001$ \\
        && 50k$^{\dag}$ & 10 & 2 & $0.0285 \pm 0.0002$ \\
    \bottomrule
    \end{tabular}
    \caption{Uniformity of forged mini-batches measured in $\ell_1$ distance between frequencies of all samples and the uniform distribution. Mean and standard errors among 100 PoRs are reported for each setting. The standard SGD yields a distance up to $0.177$. Results in all settings are much lower than this value, indicating the forged mini-batches are adequately uniform.}
    \label{tab: uniformity}
\end{table}

\subsection{Impact of Hyper-parameters}
In this section, we answer Q2. Specifically, we analyze the impact of the two key hyper-parameters of our proposed Algorithm \ref{alg: proposed}, number of splits $\K$ and number of candidate mini-batches $\M$, on the quality of forging. A large $\K$ leads to low mini-batch re-using  and a large $\M$ means more candidate mini-batches to choose from -- both come at the cost of increased computational overhead  (see Section \ref{sec:PoR:algorithm}). We evaluate for $\K \in \{5,10\}$ and $\M \in \{200,400\} $ on MNIST.
The $L_2$-distance metric $\d$ is reported in Fig. \ref{fig: model dist mnist ablation}. We observe that increasing the values of both $\K$ and $\M$ do not show any significant improvement over the values reported originally in Fig. \ref{fig: model dist mnist}. We report the correspond values for the membership prediction difference metric in the \texttt{Diff} setting in Table \ref{tab: PoR failure rate ablation}. Again we observe that the values are comparable to that of our prior results in Table \ref{tab: PoR failure rate setting 1} for most of our experimental settings. We see a slight improvement for the case of $\R=100$ and $\M=400$. This validates our choice of the hyper-parameters for our primary experiments in Section \ref{sec:practice:results}.

\subsection{Distribution of Forged Mini-batches}

In this section, we answer Q3. Stated otherwise, the goal here is to check whether a verifier can distinguish between a forged \PoL~and a true one (i.e., one recording a true SGD training process). We answer this question from a statistical perspective here -- we test whether the distribution of the samples in the forged mini-batches satisfy uniform distribution (which would be the case for a true \PoL). Specifically, we compute the $\ell_1$-distance between the frequencies of all the samples in  the forged mini-batches $\{B_{\mathrm{forge}}^{(t)}(x_i)\}_{t=1}^T$ and the uniform distribution (Table \ref{tab: uniformity}). We choose this as our metric because standard statistical tests, such as the Kolmogorov–Smirnov test, are a misfit for our setting (see Appendix \ref{app:eval} for details). We observe that $\ell_1$-distances in all of our experimental settings have a very low value indicating high uniformity in the distributions. 
As expected, a larger $\R$ results in less uniformity. Additionally, the uniformity  of the modified SGD setting is similar to that of the standard SGD setting. 
\begin{figure}[!t]
    \centering
    \includegraphics[width=0.45\textwidth]{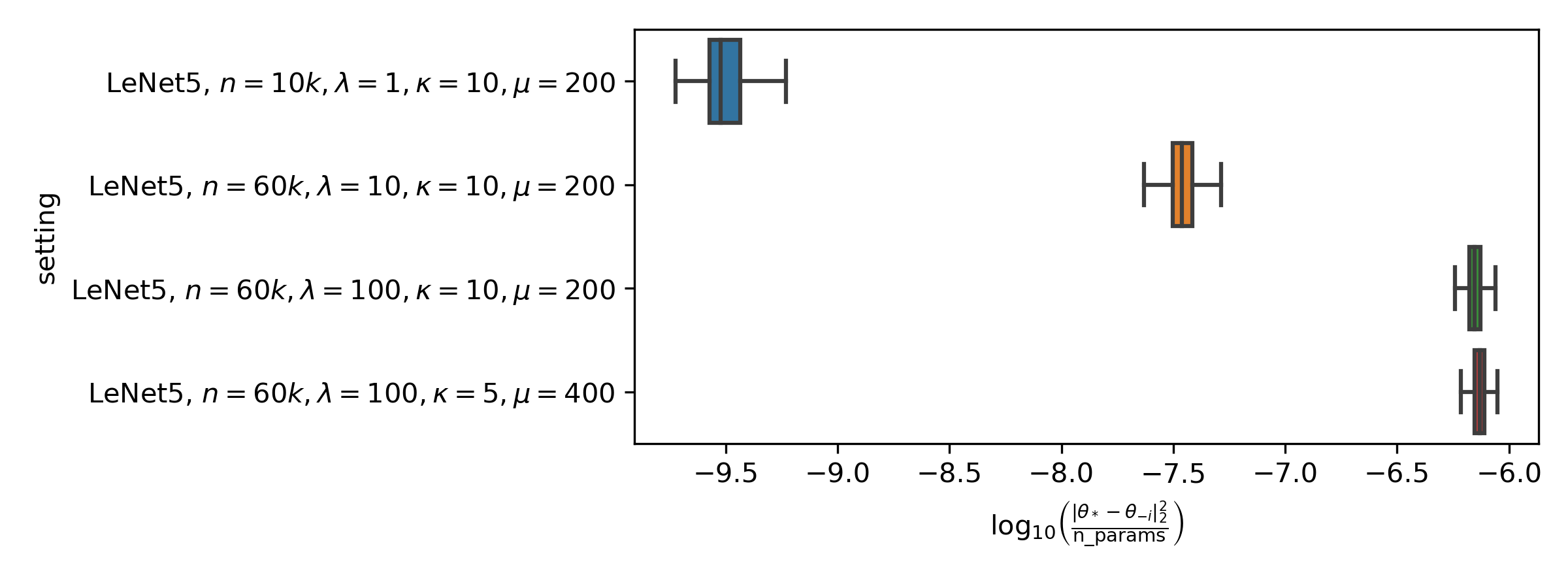}
    \caption{Box plot of $L_2$ model distances on MNIST, with different $\K$ and $\M$ values compared to Fig. \ref{fig: model dist mnist} (where $\K=5$ and $\M=200$). There is no improvement on model distances when increasing $\K$ from 5 to 10 or $\M$ from 200 to 400. }
    \label{fig: model dist mnist ablation}
\end{figure}

\begin{table*}[!t]
    \centering
    \begin{tabular}{cccccccc}
    \toprule
        \multirow{2}{*}{$n$} & \multirow{2}{*}{$\R$} & \multirow{2}{*}{$\K$} & \multirow{2}{*}{$\M$} & \multicolumn{4}{c}{\% Data points with different predictions in the \texttt{Diff} setting} \\ \cline{5-8}
        &&&& LiRA & EnhancedMIA & MEntr & Xent \\ \hline
        10k & 1   & 10 & 200 & 0.00\% & 0.01\% & 0.14\% & 0.27\% \\
        60k & 10  & 10 & 200 & 0.02\% & 0.13\% & 0.31\% & 0.31\% \\
        60k & 100 & 10 & 200 & 0.07\% & 0.65\% & 1.39\% & 1.30\% \\ 
        60k & 100 & 5  & 400 & 0.08\% & 0.56\% & 1.35\% & 1.27\% \\ 
    \bottomrule
    \end{tabular}
    \caption{Membership prediction differences in the Diff setting on MNIST, with different $\K$ and $\M$ values compared to Table \ref{tab: PoR failure rate setting 1} (where $\K = 5$ and $\M = 200$). There is a slight or no improvement on PoR~failure rates when increasing $\K$ from 5 to 10 or $\M$ from 200 to 400.}
    \label{tab: PoR failure rate ablation}
\end{table*}

\subsection{Discussion}\label{sec:practice:discussion}
The key takeaways from our evaluation are as follows.

First, we are able to generate a valid \PoR~for most of the data points. 
 Based on the above observations, we can conclude that we are able to generate membership inference equivalent \footnote{modulo our computation relaxation of testing only on a random sample of $D_{-i}$} models (Definition \ref{def:MI:equivalence}) for at least $98.8\%$ and $87.8\%$ of the data points of MNIST and CIFAR-10, respectively. In other words, we are able to generate a valid \PoR~for the same number of data points. Thus, we see that our \PoR~generation algorithm performs remarkably well in practice even for a complex dataset as CIFAR-10. 

    Second, our performance is even better on the state-of-the-art \MI~attacks that focus on a low false positive regime. The reason why this is an important observation is that there is a growing consensus that, for improving the practical reliability of membership inferencing predictions, average accuracy is not a meaningful metric for evaluating \MI~attacks~\cite{Carlini22,discredit22,Watson2021,ye2021enhanced}. Rather, a better metric is measuring the true positive rate of an \MI~attack under \textit{low} false positive rate conditions. Out of the four evaluated attacks, only the more recent LiRA and EnhancedMIA attacks are currently successful in the low false positive rate regime (LiRA has $10\times$ better performance than the other two attacks as reported in \cite{Carlini22}). Our proposed algorithm performs even better on these two attacks -- we are able to generate a valid \PoR~for $99.9\%$ and $95.7\%$ of the data points of MNIST and CIFAR-10, respectively.

  Third,  the $\R$ parameter (the number of data points that share the \textit{same} \PoR -- combination approximation) results in  a computational cost/quality of generated \PoR~trade-off. Lower the value of $\R$, better is the quality of the generated \PoR~at the cost of increased computational overhead. Specifically, the quality of the forging (equivalently, the generated \textsf{PoR}s) is the best for $\R=1$ which is studied under the sampling approximation evaluation setting.

\section{Implications of a PoR}\label{sec:implication}
In this section, we discuss the implications of a \PoR~in the context of \MI~attacks. 
\par\textbf{Consequences for the model owner.} The primary consequence of our ability to construct valid \textsf{PoR}s is that it establishes the unreliability of  \MI~attacks in practice. Using a \PoR, the model owner can discredit any membership inference claim from an adversary on the data points of its training dataset. As a result, the adversary fails to bring any real consequences for the model owner based on its prediction. For instance, say an adversary presents its membership inference claim on a data point $\x$ in a court of law. However, as discussed above,  a valid \PoR~enables the model owner to plausibly deny the \MI~prediction and present a counter claim that $\x$ is in fact \textit{not} a member. This is enough to raise a reasonable doubt on the adversary's claim. The model owner can, thus, get the case to be dismissed, thereby relieving themselves of any legal consequence. Thus, an adversary cannot ``go-to-court" with a membership inference claim.  Consequently, \textsf{PoR}s bring to question the viability and severity of any real threat a model owner faces from an \MI~attack in practice.  

It is interesting to note that in some real-world settings, the model owner's ability to generate valid \textsf{PoR}s could have adverse implications. For instance, consider a setting where an auditor wants to check whether a model used a target data point for its training without proper legal consent. This could be useful for detecting various instances of data misuse, such as privacy violation, copyright abuse or jurisdictional infringement.  The auditor can launch a \MI~attack on the model for the target data point and register a case in a court of law based on the prediction. However, as discussed above, the model owner can successfully repudiate the auditor's claim with a \PoR~and get the case to be acquitted. It is interesting to note that although the entity carrying out the attack is designated as the ``adversary" in our technical formalization, the actual semantics could be quite different in reality. In the aforementioned case, the challenger (model owner) is the unscrupulous entity who can now get away with data misuse. 

\par \textbf{Re-evaluating our understanding of \MI~attack predictions.}  \MI~attacks predictions cannot be relied upon for data points with a valid \PoR~in practice. However, as discussed in Section \ref{sec:PoR:impossibility}, we observe that it is impossible to generate a valid \PoR~for all possible data points. Specifically, one cannot generate a valid \PoR~for outliers or out-of-distribution (OOD) data points. This is also highlighted in our empirical evaluation (Section \ref{sec:practice}). Based on the above observation, we argue that our current understanding of the \textit{practical implications} of an \MI~attack predictions requires re-evaluation. We argue that a \MI~attack in perhaps better suited for distinguishing between in-lier/in-distribution and outlier/out-of-distribution data points in practice. This highlights the disparity in privacy vulnerability of the different data points  -- while the practical efficacy of an \MI~attack on in-distribution data points is questionable,  data points that are not well-represented in the dataset (such as, data points from individuals belonging to minority groups) are still susceptible to privacy violations. 
\par Similar observations have been made in prior literature as well.  Carlini et al.~\cite{Carlini22} observed that \MI~attacks are able to identify out-of-distribution data points (records) with higher success.  Ye et al.~\cite{ye2021enhanced} also observed that different data points have different vulnerabilities. Watson et al.~\cite{Watson2021} showed that the success of an \MI~attack is correlated with the model's difficulty of correctly classifying a target sample -- well-represented points have higher false positive rates. 
A concurrent work~\cite{discredit22} shows that the membership score distributions of member samples are similar to that of their non-member neighbors and argues that \MI~attacks identify the memorized sub-populations (samples that are close in the latent space) and not the individual samples (records). In other words, \MI~attacks can identify that a sample from a sub-population is a member, but they cannot reliably identify which exact sample in that sub-population is in the training dataset. 




\par \textbf{Re-evaluating the privacy leakage of \ML~models.} Conceptually for an \MI~attack, the sensitive information in contention is the single bit indicating the data point's membership status. However, a \PoR~empowers the model owner to plausibly deny any claims on this membership status. This highlights the limitations of \MI~attacks for measuring privacy leakage of \ML~models in practice.  For instance, in the real-world setting of a privacy audit discussed in the above paragraph, \MI~attacks clearly fail to capture the correct privacy semantics.  It is important to note that the validity of a \PoR~is completely agnostic of the  \MI~attack, i.e., a \PoR~is effective even if the \MI~attack makes genuinely reliable predictions (for instance, with very low false positive rate on the entire input distribution). In other words, inferring just the membership status is inadequate to make any definitive conclusions about a model's privacy leakage in many real-world settings. This urges us to re-evaluate the rationale behind using \MI~attacks as a privacy metric in practice. 

\section{Discussion}\label{sec:discussion}
In this section, we discuss several points relevant to \textsf{PoR}s.

\textbf{Zero-knowledge access to \textsf{PoR}s.} It is important to note that revealing a \textsf{PoR} in the clear to an adversary could be counterproductive for privacy. This is because the \PoR~records all the true data points in $D\setminus \{\x\}$ (in the form of the recorded mini-batches) which is revealed to the adversary, thereby violating privacy. So in practice, the adversary should be allowed a \textit{zero-knowledge} access to the \PoR. By this statement we mean that the adversary should be convinced of the validity of the \PoR~\textit{without} learning anything else about the rest of the data points of the dataset. Note that this is not a problem in the context of machine unlearning (for forgeability, Section \ref{sec:background:forgeability}) or establishing model ownership (for \textsf{PoLs}, Section \ref{sec:background:PoL}). However, privacy is innate in the context of \MI~attacks. 
 We list down a few mitigation strategies for the model owner as follows:
\begin{enumerate}\item  The model owner can use zero-knowledge proofs and achieve this goal cryptographically. Specifically, the model owner can first prove via private set intersection~\cite{} that $D'\cap \{\x\} =\emptyset$. Next, they can commit to $D'$ and show via zero-knowledge proofs that the computation of the training log was performed correctly on $D'$. \item The verification of the \PoR~log can be performed by a trusted third-party entity. For instance, as mentioned in Section \ref{sec:implication}, regulatory bodies such as the Electronic Frontier Foundation (EFF) or a judge/juror in the court of law could act as the verifier. The privacy of the data points used in the \PoR~is thereby protected via legal obligations. \item The \PoR~can be constructed from a public dataset, $D_p$, instead of from $D\setminus \{x^*\}$, such that $D_p$ is drawn from the distribution as $D$. The assumption of having a public dataset is common in both ML \cite{xie2020zeno,Cao2021FLTrustBF,kairouz2019advances} as well as privacy literature \cite{liu2021leveraging,Bassily2020PrivateQR,Beimel2020ThePO,EIFFeL}. The dataset can be small and obtained by manual labeling \cite{Google17}. Recall, that for generating a \PoR~from $D'=D\setminus \{\x\}$, all the original mini-batches $\B, t \in [\T]$ that contain the data point $\x$ are replaced with forged ones, $\Bf$, without $\x$. In order to construct the \PoR~from the public dataset, $D_p$, we need to replace \textit{all} the $\T$ original mini-batches with forged mini-batches sampled from $D_p$. Note that generating a \PoR~from $D_p$ would potentially degrade its quality, i.e., the resulting tolerance parameter $||\theta-\theta'||_2$ would be higher. This is because first, the mini-batches for every iteration has to be replaced, and second, none of the forged mini-batches contain any of the data points in $D$ resulting in a poorer approximation of the original gradient.
 \end{enumerate}
\textbf{Knowledge about the size of the dataset} In our context, an additional assumption is that the the adversary cannot know the exact size of the dataset. It is so  because with the current mechanism, the \PoR~can be only constructed from $D'\subset D$, i.e., a dataset that is strictly smaller than $|D|=n$. In case the adversary already knows $n$, then  they can ``reject" the \PoR~just based on the dataset size.  

\textbf{Spoofing attacks on \PoL.}  Recall that \PoL~was originally introduced in the context of establishing model ownership -- the central idea being only an entity who has performed the training from scratch (or has expended at least the equivalent computational effort) should be able to produce a valid \PoL. However, recent work~\cite{Spoof1, Spoof2} has shown that it is in fact possible for an adversary to \textit{spoof} a \PoL~for a model they did train with less much less computational effort than the actual model owner (trainer). This highlights the limitations of using \PoL~for resolving  model ownership. However, the existence of such spoofing attacks have no bearing in our context. It is because for a \PoR~to be valid, it is enough to just establish computational feasibility to have obtained $\mf$ from $D'$. In other words, the fact that it is plausible to have trained $\mf$ on $D'$ is enough to raise a reasonable doubt on the adversary's membership inference claim on the data point $\x$. Hence, even if the challenger ``spoofs" the \PoL~(hence \PoR) it does not affect our claim of repudiation -- the mechanism of generation of the \PoR~is  completely irrelevant to our goal of establishing computational feasibility. As discussed in Section \ref{sec:implication}, in a court-of-law, the burden-of-proof would lie on the adversary, i.e., establishing \textit{beyond} a reasonable doubt that was $\x$ indeed used by the challenger for training.

\textbf{\PoR~for a claim on non-membership.} In this paper, we explore the question of repudiating a membership claim, i.e., if $x\in D$. However, a complementary question could be when an adversary claims that a data point $x$ is \textit{not} a member, $\x\not \in D$. A real-world example of such a scenario could be when an adversary wants to claim that the model owner  has intentionally excluded some data points, say that belong to the minority population, to bias the training model. In this case, a \PoR~would have to show that the model owner could have obtained $\mf$ from the dataset $D'=D\cup\{\x\}$. Given a \PoL~for the $\langle \m,D \rangle$, we outline a possible strategy for the construction of such a \PoR. First, select a set of random iterations $H\subset [\T]$. Next, the original mini-batch $\B$ for each iteration in $t \in H$ is replaced with a forged mini-batch $\Bf$ such that $\x \in \Bf$. We surmise that as long as $\x$ is not an outlier, the above strategy would generate a valid \PoR. We leave a more detailed exploration of such a \PoR~as a part of future work.

\textbf{Roles of the challenger and adversary.} Recall that a \PoR~is essentially a forged \PoL. Thus, we leverage forgeability, a concept introduced in the context of machine unlearning, to repudiate claims on membership inferencing. Despite relying on the same technical tool (forged \textsf{PoL}s) the roles of the challenger and adversary are essayed by different parties in the two contexts. 

First, we discuss the forging security game.
\begin{defn}[Forging Security Game $\mathcal{G}_F(\cdot)$] The forging security game $\mathcal{G}_F(\cdot)$ is defined as follows: \begin{enumerate} \item Adversary $\calA$ has a dataset $D$ and a model $\m$ trained on it. \item Challenger $\calC$ selects a point $x_{-} \in D$ to be deleted and sends it to the adversary $\calA$. \item Adversary $\calA$ provides a \textit{forged} \PoL~based on Algorithm \ref{alg: forging attack} \item Challenger $\calC$ verifies the \PoL. \end{enumerate}\end{defn}

For the context of machine unlearning, the challenger  has to be somebody who knows a data point in $D$ (in practice, it is typically an individual who has contributed to $D$).   On the other hand,  the model owner is actually the adversary who is trying to be unscrupulous with a deletion request.  Thus using the above attack, the model owner can \textit{claim} to have unlearned for free. 

Now, let's look at the \PoR~security game. 
\begin{defn}[\PoR~Security Game, $\mathcal{G}_{\PoR}(\cdot)$] The \PoR~security game is defined as follows: \begin{enumerate}\item Challenger $\calC$ holds a private dataset $D$ and a model $\m$ trained on it. The model is publicly released.
\item Adversary $\mathcal{A}$ \textit{correctly} predicts that a data point $x^* \in D$, i.e., $\calA(\x,\m,\Psi)=1$, and sends this membership inferencing claim to the challenger $\calC$.
\item Challenger $\calC$ responds with a Proof-of-Repudiation, $\PoR$.
\item Adversary $\calA$ verifies the \PoR.
\end{enumerate}
\end{defn}
As discussed in Section \ref{sec:background:MI}, in this paper we focus on \MI~attacks in practice. Hence, 
 for a membership inferencing, the model owner is the challenger who is trying to fend themselves from an adversary carrying out an \MI~attack. Thus, the role of the model owner is switched in the two contexts. 
\section{Related Work}\label{sec:relatedwork}
Membership inference attacks predict whether a particular data point was used for training a model 
(\cite{homer2008resolving, sankararaman2009genomic,dwork2015robust, dwork2017exposed, shokri2017membership,carlini2019secret, murakonda2020ml,song2020introducing, carlini2021membership,carlini2021extracting,ye2021enhanced}). 
There is a connection between membership inferencing and other topics in \ML~privacy, such as differential privacy (\cite{carlini2019secret, jagielski2020auditing,humphries2020differentially, nasr2021adversary}), memorization (\cite{feldman2020does, zhang2021understanding}), and over-fitting (\cite{yeom2018privacy}).


\par Machine unlearning is the task to delete a data point from a learned model and approximate machine unlearning  outputs a model that is statistically indistinguishable from the model obtained by deleting and retraining (\cite{cao2015towards,ginart2019making,guo2019certified,schelter2020amnesia,borassi2020sliding,ullah2021machine,neel2021descent,sekhari2021remember,izzo2021approximate,kong2022approximate}). The recently proposed forging attack in \cite{Thudi22} challenges the foundation of machine unlearning by generating a forged \PoL~log which is indistinguishable from the correct one (dataset after deleting). In this paper, we have showed a connection between membership inference attacks and forging attacks. 

Concurrent work by Rezaei et al.~\cite{discredit22} has also studied the practical unreliability of \MI~attacks. However, their approach is completely different from our concept of \textsf{PoR}s. They consider an auditing scenario where an auditor carries out an \MI~attack that claims to have a low false positive rate.  Given a list of data points that the attack claimed as members, the auditee (model owner) crafts a "discrediting dataset" consisting of non-members data points that are mis-predicted by the \MI~attack as members (false positive error). This challenges the reliability of the \MI~predictions. 

It is important to note that there is a fundamental difference between prior literature on defenses against \MI~attacks~\cite{Defense1,Nasr18,distillation,MemGuard} and our setting. Defenses are designed to reduce the efficacy (success rate) of \MI~attacks. The key insight is to induce similar model output distribution on both the training (member) and testing (non-member) dataset and is achieved via either training time techniques (such as dropout~\cite{Defense1}, $\ell_2$-norm regularization~\cite{shokri2017membership},  model stacking~\cite{smoothing}, min-max adversary regularization~\cite{Nasr18}, differential privacy \cite{DP1,DP2}, early stopping~\cite{MEntr}, knowledge distillation~\cite{distillation}) and or inference time mechanisms (such as output perturbation \cite{MemGuard}).  However, our setting is completely different -- a \PoR~enables a model owner to discredit a membership inference claim \textit{post-attack}.   The generation of a \PoR~requires \textit{no} modifications for pre-computing the forged mini-batches only to reduce computational cost of generating the \textsf{
PoR}s for all the data points in $D$.  to the model's training or inference pipeline. In other words, a \PoR~makes an \MI~attack ineffective in practice even against a \ML~model that has been already trained or is publicly available.

\section{Conclusion}\label{sec:conclusion}
In this paper, we investigate the reliability of membership inference in practice. Specifically, we study if it is possible for a model owner to refute a membership inference claim by introducing the notion of Proof-of-Repudiation (\PoR). A \PoR~enables a model owner to plausibly deny the membership inference prediction and discredit the predictions of an \MI~attack. Consequently, an adversary cannot ``go-to-court" with the prediction of an \MI~attack. The concrete construction is based on the concept of forgeability introduced by Thudi et al. \cite{Thudi22}. Our empirical evaluations have shown that it is possible to construct valid \textsf{PoR}s~efficiently in practice.

\noindent \textbf{Acknowledgements.} This material is partially based upon work supported by the National Science Foundation under Grant \# 2127309 to the Computing Research Association for the CIFellows Project. Kamalika Chaudhuri and Zhifeng Kong would like to thank NSF under CNS 1804829, ONR under N00014-20-1-2334 and ARO MURI W911NF2110317 for research support.

\bibliographystyle{plain}
\bibliography{references}
\clearpage
\onecolumn
\appendix
\subsection{Notations}
Table \ref{tab:notations} describes the notations used in the paper.

\begin{table}[!h]
\centering 
\caption{Notations used in the paper.}
\label{tab:notations}
\begin{tabular}{ rl } 
\toprule
Notation & Meaning \\ \hline
 $\D$ & Input data distribution\\
 $D$ & Original Dataset \\ 
$D'$ & Forging dataset \\ 
$\mathcal{A}$ & Adversary/attack\\
$\calC$ & Challenger\\
 $\x$ &  \MI~claim made on this point\\
 $\M$& Number of candidate batches\\
 $\R$ & number of \textsf{PoR}s combined \\
 $\T$ & Number of iterations\\
 $\m$ & Original model\\
 $\mf$ & Forged model\\
 $\theta$ & Parameters of the original model\\
 $\theta'$ & Parameters of the forged model\\
 $n$ & Dataset size $|D|$\\
 $\B$ & Original mini-batch for $t$-th iteration\\
  $\Bf$ & Forged mini-batch for $t$-th iteration\\
 $b$ & Mini-batch size\\
 \bottomrule
\end{tabular}
\end{table}

\subsection{The Complete PoR~Generation Algorithm}\label{app:algorithm}

In Algorithm \ref{alg: complete}, we present the complete \PoR~generation algorithm for models trained with the Random Horizontal Flip data augmentation and modified SGD with momentum, weight decay, and learning rate scheduling. For each mini-batch $B$, we let $\texttt{Flip}(B)\in\{0,1\}^{|B|}$ be the corresponding binary flag indicating whether each data point has been flipped horizontally at randomly, and $\texttt{Aug}(\texttt{Flip}(B), B)$ be the mini-batch after data augmentation. 

\begin{algorithm*}
    \centering
    \caption{Generation of PoRs (Complete)}
    \label{alg: complete}
    \begin{algorithmic}[1]
        \State \textbf{Inputs:} $D=\{x_1,\cdots,x_n\}$ - Training dataset;  $\T$ - Total number of iterations; $B_*^{(t)}$ - Training mini-batches for the $t$-th iteration where $t\in [\T]$; $\texttt{Flip}(B_*^{(t)})$ - Binary flags indicating whether random horizontal flip augmentation has been applied to the samples in $B_*^{(t)}$; $b$ - Size of each mini-batch, $\left|B_*^{(t)}\right|$; $\theta^{(0)}$ - Initialization parameters.
        \State \textbf{Hyper-parameters:} $\texttt{step\_size}^{(t)}$ - Learning rate at step $t$; $\texttt{momentum}$ - Momentum coefficient; $\texttt{weight\_decay}$ - Weight decay coefficient; $\K$ - Number of splits; $\M$ - Number of candidate mini-batches sampled every iteration for forging.
        \For{$t=1,\cdots,\T$}
            \State Compute gradients and update the original model: \\
            $$\texttt{grad}^{(t)} = \nabla_{\theta} \left(\frac1b\sum_{x\in \texttt{Aug}(\texttt{Flip}(B_*^{(t)}), B_*^{(t)})} \ell(x;\theta^{(t)}) + \texttt{weight\_decay}\cdot\|\theta^{(t)}\|_2^2\right),$$
            $$\theta^{(t)}=\theta^{(t-1)}-\texttt{step\_size}^{(t)}\cdot(\texttt{grad}^{(t)}+\texttt{momentum}\cdot\texttt{velocity}^{(t-1)}).$$
            \State Randomly split $D=D_1^{(t)}\cup\cdots\cup D_\K^{(t)}$ that satisfy Eqs. \eqref{eq: rand split condition 1} and \eqref{eq: rand split condition 2}.
            \For{$k=1,\cdots,\K$}
                \State Randomly select subset $D_{\mathrm{sub}}$ from $D\setminus D_k^{(t)}$.
                \State Randomly select $\M$ mini-batches with size $b$ from $D_{\mathrm{sub}}$: $B_1^{(k,t)},\cdots B_\M^{(k,t)}$.
                \For {$m=1,\cdots,\M$}
                    \State Sample $\texttt{Flip}(B_m^{(k,t)})\in\texttt{Uniform}(\{0,1\}^{b})$.
                \EndFor
                \State Choose the best mini-batch and augmentation variables to approximate the gradient: \\
                $$m_*^{(k,t)} = \arg\min_{m \in [\M]} \left\|\nabla_{\theta} \left(\frac1b\sum_{x\in \texttt{Aug}(\texttt{Flip}(B_m^{(k,t)}), B_m^{(k,t)})}\ell(x;\theta^{(t)})+ \texttt{weight\_decay}\cdot\|\theta^{(t)}\|_2^2\right)-\texttt{grad}^{(t)}\right\|_2^2.$$
                \For{{$x_i\in D_k^{(t)}$}}
                    \If{{$x_i\in B_*^{t}$}}
                        \State {$B_{\texttt{forge}}^{(t)}(x_i) = B_{m_*^{(k,t)}}^{(k,t)}$, $\texttt{Flip}(B_{\texttt{forge}}^{(t)}(x_i)) = \texttt{Flip}(B_{m_*^{(k,t)}}^{(k,t)})$;}
                    \Else 
                        \State {$B_{\texttt{forge}}^{t}(x_i) = B_*^{(t)}$, $\texttt{Flip}(B_{\texttt{forge}}^{(t)}(x_i)) = \texttt{Flip}(B_*^{(t)})$.}
                    \EndIf 
                \EndFor
            \EndFor
        \EndFor
        \State \textbf{Return:} original model $\theta_*=\theta^{(\T)}$; all forged mini-batches $\{B_{\texttt{forge}}^{(t)}(x_i):  i\in [n], t \in [\T]\}$; all augmentation variables $\{\texttt{Flip}(B_{\texttt{forge}}^{(t)}(x_i)):  i\in [n], t \in [\T]\}$.
    \end{algorithmic}
\end{algorithm*}

\subsection{Experimental Configuration and Additional Results}\label{app:eval}

\paragraph{Models}~~\\

For MNIST, we use the standard LeNet5 implementation based on \url{https://github.com/erykml/medium_articles}. It has 61.7K parameters. We use the standard SGD with learning rate $=1\times10^{-2}$ and a batch size of 100 to train the model for 20 epochs.

For CIFAR10, we use the VGG-mini and ResNet-mini implementations based on \url{https://github.com/nikhilbarhate99/Image-Classifiers}. VGG-mini has 5.75M parameters and ResNet-mini has 1.49M parameters. We first use the standard SGD with learning rate $=1\times10^{-2}$ and a batch size of 100 to train the model for 20 epochs. We also use a modified SGD to train for 20 epochs, whose hyper-parameters are the same with additionally a weight decay of $5\times10^{-4}$, momentum of $0.9$, and cosine learning rate scheduler based on \url{https://pytorch.org/docs/stable/generated/torch.optim.lr_scheduler.CosineAnnealingLR.html}.

~~\newline

\paragraph{MI~Attacks}~~\\

We implement the LiRA attack based on \url{https://github.com/tensorflow/privacy/tree/master/research/mi_lira_2021}. For MNIST we train 16 \texttt{cnn32-3-max} shadow models with learning rate $=0.001$ and a batch size of 64 for 20 epochs. For CIFAR10 we train 16 \texttt{wrn28-2} shadow models with learning rate $=0.1$ and a batch size of 256 for 100 epochs. The \MI~score threshold is computed based on maximum accuracy.

We implement the EnhancedMIA attack based on \url{https://github.com/privacytrustlab/ml_privacy_meter/blob/master/tutorials/population_metric.ipynb}. The \MI~score threshold is computed to let the false positive rate $=10\%$.

We implement the MEntr and Xent attacks based on \url{https://github.com/inspire-group/membership-inference-evaluation}. We implement the MIDA attack based on \url{https://github.com/yigitcankaya/augmentation_mia}. The label-specific \MI~score thresholds are computed based on maximum accuracy.

~~\newline

\paragraph{Metrics}~~\\

The \MI~prediction difference metric is measured in three settings: \texttt{Diff}, \texttt{Common}, and \texttt{Validation}. In the \texttt{Common} setting, as it is computationally expensive to look at all points in $D_{-i}$, we look at a random subset instead. The total number of \MI~predictions checked is $5n$. We do the similar in the \texttt{Common} setting. 

The additional \MI~score difference metrics are shown in the box plots below. Outliers are not displayed in the box plots.  

\begin{figure}[!h]
\centering
\includegraphics[trim=0 10 0 0, clip, width=0.4\textwidth]{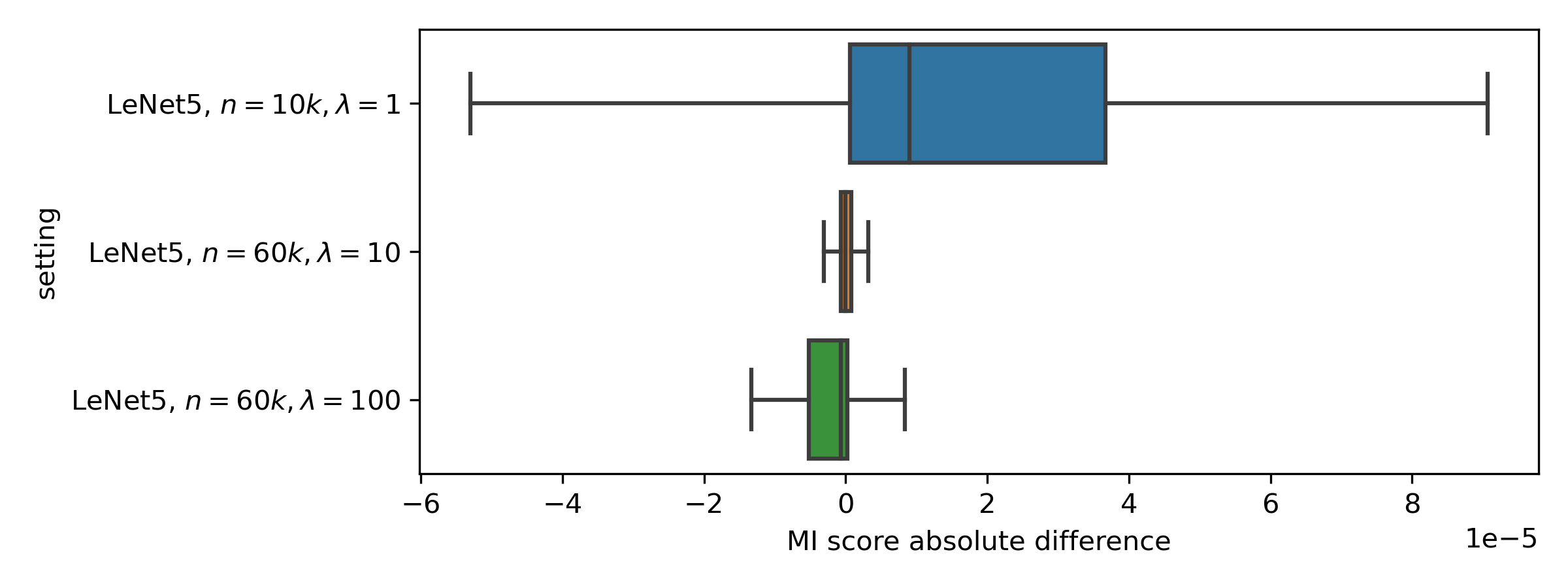}
\caption{Box plots of MEntr score differences ($\s$) on MNIST.}
\end{figure}

\begin{figure}[!h]
\centering
\includegraphics[trim=0 10 0 0, clip, width=0.4\textwidth]{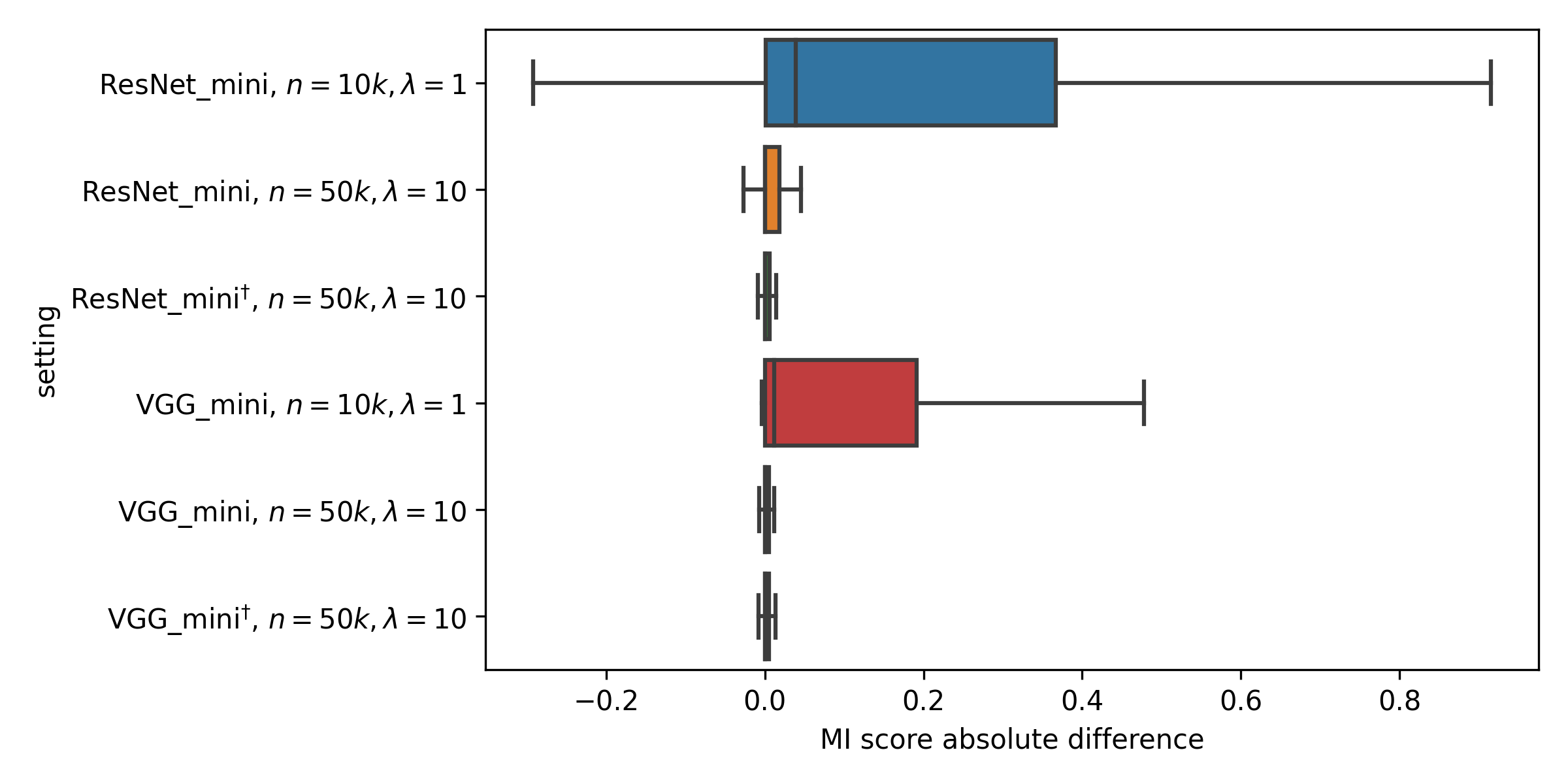}
\includegraphics[trim=0 10 0 0, clip, width=0.4\textwidth]{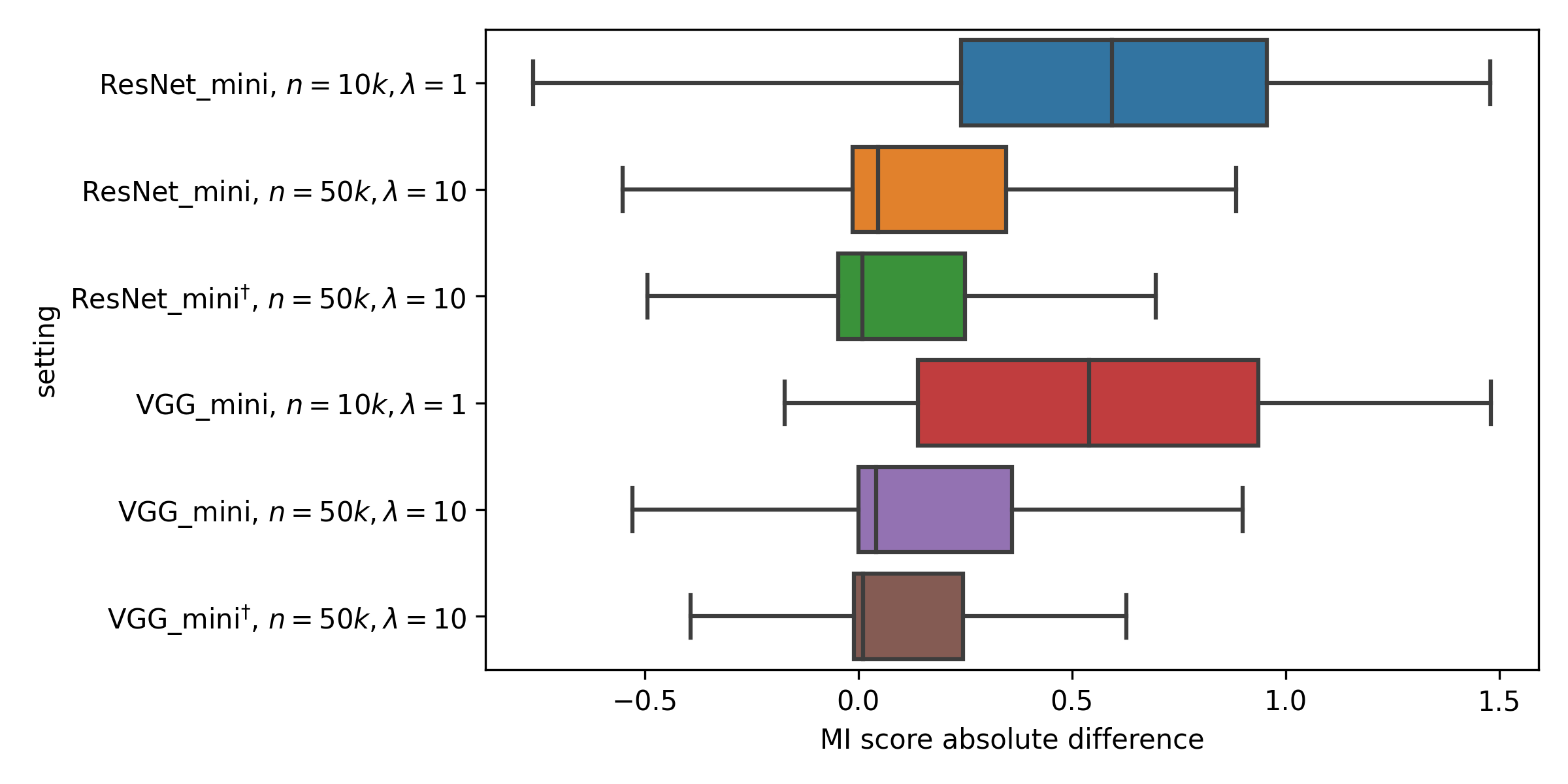}
\caption{Box plots of Mentr (left) and MIDA (right) score differences ($\s$) on CIFAR10 ($\dag$ means the modified SGD is used).}
\end{figure}

\newpage
\subsection{Indistinguishability of Mini-batches}\label{app:indistinguishability}
Here, we describe why standard two-sample statistical tests, such as the Kolmogorov-Smirnov, test are not suitable for distinguishing between the forged mini-batches and the ones obtained from the original training (following standard SGD). The reason behind this is that mini-batches trained under standard SGD do not follow   any fixed distribution and rather depends on the specificities of the exact practical implementation of SGD. For example, if we randomly sample mini-batches with no replacement (i.e. a mini-batch has to contain different samples), the frequency distribution is shown as the yellow curve in Fig. \ref{fig: batch visualization}. On the other hand, if the SGD is implemented as a for loop of the dataloader in PyTorch, then the sample frequencies are uniform, as shown in the green curve in Fig. \ref{fig: batch visualization}. Based on our empirical evaluation, the forged mini-batches have a frequency distribution (shown in the blue curve in Fig. \ref{fig: batch visualization}  which lies in between the two aforementioned curves. Hence, running any two-sample test will almost certainly output a negative result (i.e., the two distributions are not the same). However, the forged mini-batches are very uniform and it is plausible that there exists a specific implementation of SGD -- for example, a for loop implementation corresponding to a distributed training setting -- that yields a similar distribution. Therefore, in this paper, we conjecture that one cannot disprove the validity of any distribution (i.e., the distribution corresponds to a valid SGD training) that lies between the green and yellow curves. 

\begin{figure}[!h]
    \centering
    \includegraphics[width=0.45\textwidth]{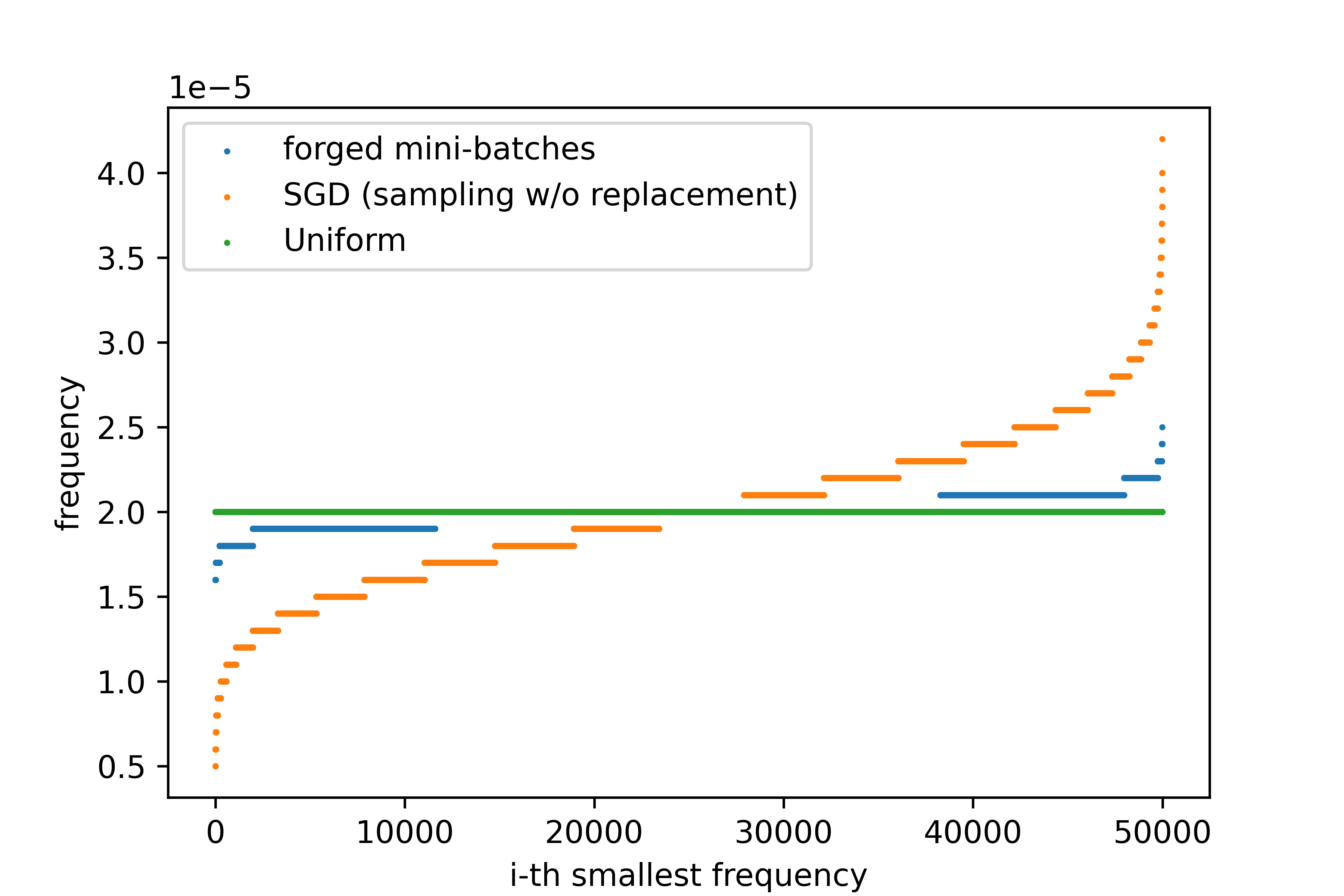}
    \caption{Sorted sample frequencies of mini-batches from various processes.}
    \label{fig: batch visualization}
\end{figure}
\end{document}